\documentclass[10pt,twocolumn,letterpaper]{article}

\usepackage{iccv}
\usepackage{times}
\usepackage{epsfig}
\usepackage{graphicx}
\usepackage{amsmath}
\usepackage{amssymb}
\usepackage{amsthm}
\usepackage[]{algorithm2e}

\newtheorem{theorem}{Theorem}
\newtheorem{cor}{Corollary}
\newtheorem{lemma}{Lemma}
\def\rank{{\mathrm{ rank}}}
\def\C{{\mathcal{C}}}
\def\N{{\mathcal{N}}}
\def\D{{\mathcal{D}}}
\DeclareMathOperator*{\argmin}{arg\,min}
\newcommand{\skal}[1]{\langle #1 \rangle}

\DeclareMathOperator{\diag}{diag}

\usepackage{booktabs}
\usepackage[table]{xcolor}
\usepackage{array,multirow,graphicx}
\usepackage{float}

\makeatletter
\@namedef{ver@everyshi.sty}{} 
\makeatother
\usepackage{tikz,pgfplots}
\pgfplotsset{compat=1.14}

\usepackage{hyperref}

 \iccvfinalcopy 

\def\iccvPaperID{2241} 

\def\reg{\mathcal{R}}
\def\A{\mathcal{A}}

\def\sigmavec{{\bf \sigma}}
\newcommand{\bX}{\bar{X}}
\newcommand{\bB}{\bar{B}}
\newcommand{\bC}{\bar{C}}
\newcommand{\tX}{\tilde{X}}
\newcommand{\tB}{\tilde{B}}
\newcommand{\tC}{\tilde{C}}

\newcommand{\fr}[2]{\frac{#1}{#2}}

\usepackage{bm}
\renewcommand{\vec}[1]{%
  \ifcat\noexpand#1\relax 
    \bm{#1} 
  \else
    \mathbf{#1} 
  \fi
}
\newcommand{\hada}{\odot}
\usepackage{bbm}
\newcommand{\R}{\mathbbm{R}}

\usepackage{mathtools}
\newcommand{\norm}[1]{\left\|#1\right\|}
\newcommand{\T}{T}

\newcommand{\mat}[1]{#1}

\ificcvfinal\fi
\begin{document}

\title{Bilinear Parameterization For Differentiable Rank-Regularization}

\def\ww{5mm}
\author{Marcus Valtonen \"Ornhag$^1$ \hspace{\ww} Carl Olsson$^{1,2}$ \hspace{\ww} Anders Heyden$^1$\\[0.2 cm]
	\begin{minipage}[c]{0.4\textwidth}
		\centering
		${}^1$Centre for Mathematical Sciences\\
		Lund University 
	\end{minipage}
	\begin{minipage}[c]{0.4\textwidth}
	\centering
	${}^2$Department of Electrical Engineering\\
	Chalmers University of Technology 
\end{minipage}
	\\[0.4cm]
	{\tt\small  \{marcusv,\,calle,\,heyden\}@math.lth.se}
}

\maketitle

\begin{abstract}
Low rank approximation is a commonly occurring problem in many computer vision and machine learning applications.
There are two common ways of optimizing the resulting models. Either the set of matrices with a
given rank can be explicitly parametrized using a bilinear factorization, or low rank can be
implicitly enforced using regularization terms penalizing non-zero singular values. While the
former approach results in differentiable problems that can be efficiently optimized using local
quadratic approximation, the latter is typically not differentiable (sometimes even discontinuous) and requires first order subgradient or splitting methods. It is well known that gradient based methods exhibit slow convergence for ill-conditioned problems.

In this paper we show how many non-differentiable regularization methods can be reformulated into smooth objectives using bilinear parameterization. This allows us to  use standard second order methods, such as Levenberg--Marquardt (LM) and Variable Projection (VarPro), to achieve accurate solutions for ill-conditioned cases. 
We show on several real and synthetic experiments that our second order formulation converges to substantially more accurate solutions than competing state-of-the-art methods.
\end{abstract}

\section{Introduction}

Low rank models have been applied to numerous vision applications ranging from high level shape and deformation to pixel appearance models \cite{tomasi-kanade-ijcv-1992,bregler-etal-cvpr-2000,yan-pollefeys-pami-2008,garg-etal-cvpr-2013,basri-etal-ijcv-2007,garg-etal-ijcv-2013,wang-etal-2012,canyi-etal-cvpr-2014}.
When the sought rank is known, a commonly occurring formulation is the least squares minimization
\begin{equation}
\min_{\rank(X)\leq r} \|\A X-b\|^2,
\label{eq:fixedrankA}
\end{equation}
where $\A:\mathbb{R}^{m \times n} \rightarrow \mathbb{R}^p$ is a linear operator, and $\|\cdot\|$ is the standard Euclidean vector norm.
In general, this is a difficult non-convex problem and some versions are even known to be NP-hard \cite{gillis-glineur-siam-2011}.
In structure from motion, a popular approach \cite{buchanan-fitzgibbon-cvpr-2005} is to optimize over a bilinear factorization $X=BC^T$, where $B$ is $m \times r$ and $C$ is $n\times r$, and solve
\begin{equation}
\min_{B,\,C} \|\A BC^T-b\|^2.
\label{eq:knownrankbilin}
\end{equation}
Since the rank is bounded by the number of columns in $B$ and $C$ this approach explicitly parametrizes the set of matrices of rank $r$. 
While bilinear approaches often perform well \cite{hong-fitzgibbon-cvpr-2015,eriksson-hengel-pami-2012} they can have local minima~\cite{buchanan-fitzgibbon-cvpr-2005}. 
Recent works \cite{hong-fitzgibbon-cvpr-2015,hong-etal-eccv-2016,hong-etal-cvpr-2017,hong-zach-cvpr-2018} have, however, shown that properly implemented, LM and VarPro approaches are remarkably robust to local minima, achieve quadratic convergence and give impressive reconstruction results. Recently \cite{ge-etal-nips-2016,bhohanapali-nips-2016,ge-etal-arxiv-2017} was able to give conditions which guarantee that there are no "spurious" local minimizers (meaning that all local minimizers are close to or identical to the global solution). They use the notion of restricted isometry property (RIP) \cite{recht-etal-siam-2010} which assumes that the operator $\A$ fulfills 
\begin{equation}
(1-\delta_r)\|X\|_F^2 \leq \|\A X\|^2 \leq (1+\delta_r)\|X\|_F^2,
\label{eq:RIP}
\end{equation}
with $0\leq\delta_r<1$, if $\rank(X) \leq r$. If the isometry constant $\delta_r$ is sufficiently small \cite{ge-etal-nips-2016,ge-etal-arxiv-2017,bhohanapali-nips-2016} prove that every local minimizer is optimal (or near optimal). Similarly, for the matrix completion problem \cite{ge-etal-arxiv-2017} showed that there are no spurious local minima under uniformly distributed missing data. While the above theoretical assumptions generally do not hold for computer vision problems such as structure from motion, these results still give some intuition as to why bilinear parameterization often works well.

An alternative approach is to optimize directly over the entries of $X$ and enforce low rank using regularization terms. Applying a robust function $f$ to the singular values $\sigma_i(X) = 1,\ldots,N=\min(m,n)$ results in a low-rank inducing objective
\begin{equation}
\min_X \reg(X)+\|\A X-b\|^2,
\label{eq:generalregformulation}
\end{equation}
where
$
\reg(X) = \sum_{i=1}^{N} f(\sigma_i(X)).
$
Besides controlling the rank of the solution the generality of the function $f$ offers increased modeling capability compared to \eqref{eq:fixedrankA} and can for example be used to add priors on the size of the non-zero singular values.

The most popular regularization approach is undoubtedly the nuclear norm, $f(\sigma_i(X))=\sigma_i(X)$, due to its convexity \cite{fazel-etal-acc-2015,recht-etal-siam-2010,oymak2011simplified,candes-etal-acm-2011,candes2009exact}. 
Under the RIP assumption exact or approximate recovery with the nuclear norm can then be guaranteed \cite{recht-etal-siam-2010,candes2009exact}.
On the other hand, since it penalizes large singular values, it suffers from a shrinking bias \cite{cabral-etal-iccv-2013,canyi-etal-cvpr-2014,larsson-olsson-ijcv-2016}. Ideally $f$ should penalize small singular values (assumed to stem from measurement noise) harder than the large ones. Therefore  non-increasing derivatives on $[0,\infty)$, or concavity, has been shown to give stronger relaxations \cite{oymak-etal-2015,mohan2010iterative,hu-etal-pami-2013,oh-etal-pami-2016,canyi2015,toh-yun-2010,gu-2016}. These non-convex formulations usually only come with local convergence guarantees. Two exceptions are \cite{larsson-olsson-ijcv-2016,olsson-etal-iccv-2017} which gave optimality guarantees for \eqref{eq:generalregformulation} with $f=f_\mu$ as in \eqref{eq:fmu}.

The regularization term is generally not differentiable as a function of $X$. Thus, optimization methods based on local quadratic approximation become infeasible. 
Figure~\ref{fig:objfuns} gives a simple illustration on a 1-dimensional example of how non-differentiability occurs at the origin. In addition it is well known that the singular values become non-differentiable functions of the matrix elements when they are non distinct. To circumvent these issues subgradient and splitting methods are often employed \cite{canyi2015,toh-yun-2010,gu-2016,nie-2012,larsson-olsson-ijcv-2016}.
It is well known from basic optimization theory (\eg{}~\cite{boyd-vandenberghe-2004}) that gradient based methods exhibit slow convergence for ill-conditioned problems.
It has also been observed (\eg{}~\cite{boyd-etal-2011}) that splitting methods rapidly reduce the objective value the first couple of iterations, while convergence to the exact solution can be slow. In this paper we show that there are computer vision problems where these approaches make very little improvements at all, returning a solution that is far from optimal. In contrast, bilinear formulations with either LM or VarPro can be made to yield accurate results in few iterations \cite{hong-fitzgibbon-cvpr-2015}.

An alternative approach that unifies bilinear parameterization with regularization approaches is based on the observation \cite{recht-etal-siam-2010} that the nuclear norm $\|X\|_*$ of a matrix $X$ can be expressed as
$
\|X\|_* = \min_{BC^T = X} \frac{\|B\|_F^2+\|C\|_F^2}{2}.
$
Thus when $f(\sigma_i(X)) = \mu \sigma_i(X)$, where $\mu$ is a scalar controlling the strength of the regularization, optimization of \eqref{eq:generalregformulation} can be formulated as 
\begin{equation}
\min_{B,C} \mu \frac{\|B\|_F^2+\|C\|_F^2}{2}+ \|\A BC^T - b\|^2.
\label{eq:nuclearbilin}
\end{equation} 
Optimizing directly over the factors has the advantages that the number of variables is much smaller and one may add constraints if a particular factorization is sought. Surprisingly, while 
\eqref{eq:nuclearbilin} is non-convex, using the convexity 
of the underlying regularization problem \eqref{eq:generalregformulation} it can be shown that any local minimizer $B$,$C$ with $\rank(B C^T) < k$, where $k$ is the number of columns in $B$ and $C$, is globally optimal \cite{bach-arxiv-2013,haeffele-vidal-arxiv-2017}. 
Additionally, the objective function is two times differentiable and second order methods can be employed. 
\begin{figure*}
\centering
\def\arraystretch{0.75}
\def\w{27mm}
\begin{tabular}{cccccc}
SCAD \cite{fan2001variable}: & Log \cite{friedman-2012}: & MCP \cite{zhang2010nearly}: & ETP \cite{gao-etal-AAAI-2011}: & Geman \cite{geman-yang-1995}: \\
\includegraphics[width=\w]{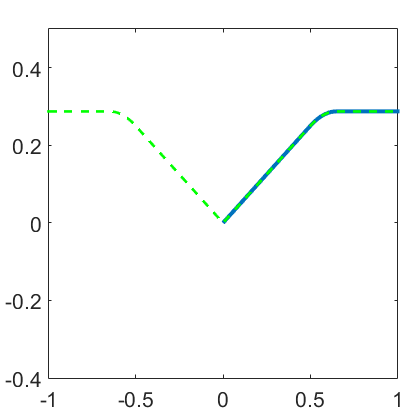} &
\includegraphics[width=\w]{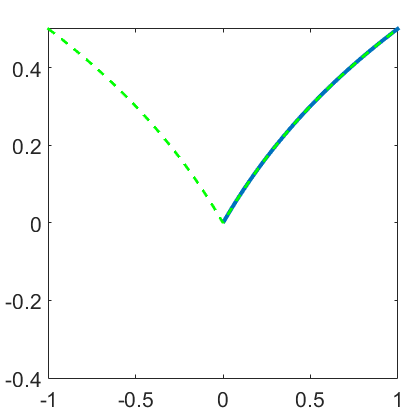} &
\includegraphics[width=\w]{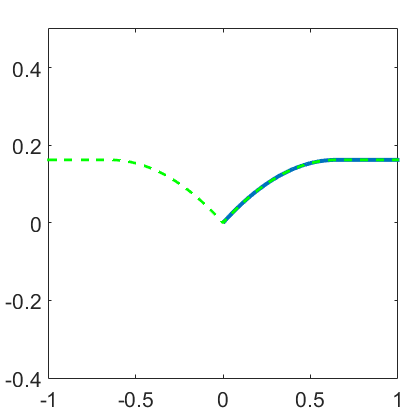} &
\includegraphics[width=\w]{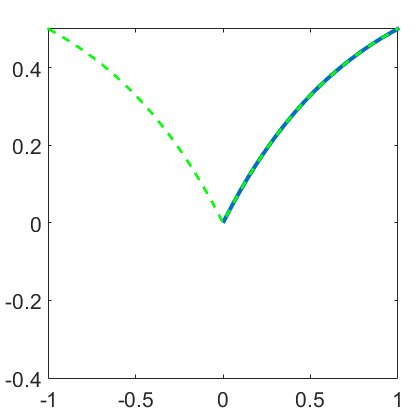} &
\includegraphics[width=\w]{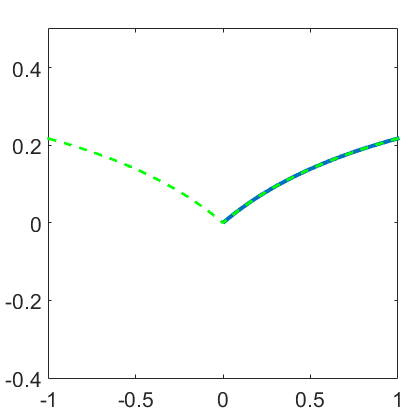}
\end{tabular}
\caption{A few commonly occurring robust penalties of the form $f(\sigma)$, with $\sigma \in [0,\infty)$ and $f$ differentiable everywhere (blue graph). 
The green dashed graph shows how non-differentiability occurs at the origin when applying the penalty to a $1 \times 1$ matrix $x\in \mathbb{R}$. In this case $\sigma(x)=|x|$ and therefore $f(\sigma(x)) = f(|x|)$. Note also that \eqref{eq:fmu} is a special case of MCP.}
\label{fig:objfuns}
\end{figure*}

In this paper we develop new regularizing terms that, similar to \eqref{eq:nuclearbilin}, work on the bilinear factors. However, in contrast to previous approaches we investigate formulations that exhibit less shrinking bias and go beyond convex penalties. Specifically, we prove that $\reg(X)  = \min_{X=BC^T} \tilde{\reg}(B,C)$, where 
\begin{equation}
\tilde{\reg}(B,C) = \sum_{i=1}^{k} f\left(\frac{\|B_i\|^2+\|C_i\|^2}{2}\right),
\end{equation}
$k$ is the number of columns, and $B_i$ and $C_i$ are the $i$:th columns of $B$ and $C$, respectively. The result holds for a general class of concave penalty functions $f$, a few of which are illustrated in Figure~\ref{fig:objfuns}. 
In view of the above result, we propose to minimize  
\begin{equation}
\tilde{\reg}(B,C)+\|\A B C^T - b\|^2.
\label{eq:generalbilin}
\end{equation}
Rather than resorting to splitting or subgradient methods we present an algorithm that uses a quadratic approximation of the objective.
Under the assumption that $f$ is differentiable, we show that
our quadratic approximation reduces to a weighted version of \eqref{eq:nuclearbilin} to which we can apply VarPro.
We show on several computer vision problems that our approach outperforms state-of-the-art methods such as \cite{shang-etal-2018,canyi2015,toh-yun-2010,gu-2016,boyd-etal-2011}.

While our problem is non-convex (both in the $X$ parameterization \eqref{eq:generalregformulation} and in the $B$, $C$ parameterization \eqref{eq:generalbilin}) we show that in some cases it is still possible to give global optimality guarantees. 
Building on the results of \cite{olsson-etal-iccv-2017} we characterize the local minima of the new formulation with the choice 
\begin{equation}
f(x) = f_\mu(x) := \mu - \max(\sqrt{\mu}-x,0)^2.
\label{eq:fmu}
\end{equation}
Specifically, for this choice, we give conditions that ensure that when a RIP constraint~\cite{recht-etal-siam-2010} holds a local minimizer of \eqref{eq:generalbilin} is a global solution of both
\begin{equation}
\min_{\rank(X)\leq r} \reg(X) + \|\A X - b\|^2,
\label{eq:regminprobl}
\end{equation}
where $\reg(X) = \sum_i f_\mu(\sigma_i(X))$, and 
\begin{equation}
\min_{\rank(X)\leq r} \mu \rank(X) + \|\A X - b\|^2.
\label{eq:rankminprobl}
\end{equation}
In summary our main contributions are:
\begin{itemize}
\item A new stronger non-convex regularization term for bilinear parameterizations with less/no shrinking bias.
\item A new iteratively reweighed VarPro algorithm optimizing accurate quadratic approximations.
\item Theoretical conditions that guarantee optimal recovery under the RIP constraint.
\item An experimental evaluation that shows that our methods outperforms state-of-the-art methods on several real computer vision problems.
\end{itemize}

\subsection{Related Work}
Our work is very much inspired by a recent series of papers by Hong \etal \cite{hong-fitzgibbon-cvpr-2015,hong-etal-eccv-2016,hong-etal-cvpr-2017,hong-zach-cvpr-2018} which show that bilinear formulations can be made remarkably robust to local minima, and achieve impressive reconstruction results for uncalibrated structure from motion problems, using the so called VarPro method.  
Our work represents an attempt to unify this line of work with regularization based alternatives, leveraging the benefits of them both.

An approach that is closely related to ours is that of \cite{cabral-etal-iccv-2013} which uses \eqref{eq:nuclearbilin} to unify the use of a regularized objective and factorization. They show that if the obtained solution has lower rank than its number of columns it is globally optimal. In practice \cite{cabral-etal-iccv-2013} observes that the shrinking bias of the nuclear norm makes it too weak to enforce a low rank when the data is noisy. Therefore, a ``continuation'' approach where the size of the factorization is gradually reduced is proposed. While this yields solutions with lower rank, the optimality guarantees no longer apply. 

Bach \etal \cite{bach-arxiv-2013} showed that
\begin{equation}
\|X\|_{s,t}:=\min_{X=BC^T} \sum_{i=1}^k\frac{\|B_i\|_s^2 + \|C_i\|_t^2}{2},
\label{eq:decomposition-norm}
\end{equation}
is convex for any choice of vector norms $\|\cdot\|_s$ and \mbox{$\|\cdot\|_t$}.
In \cite{haeffele-vidal-arxiv-2017} it was shown that a more general class of 2-homogeneous factor penalties result in a convex regularization similar to \eqref{eq:decomposition-norm}. The property that a local minimizer $B$, $C$ with $\rank(B C^T) < k$, is also extended to this case. Still, because of convexity, it is clear that these formulations will suffer from a similar shrinking bias as the nuclear norm.
Shang \etal \cite{shang-etal-2018} showed that penalization with the Schatten semi-norms $\|X\|_q = \sqrt[q]{\sum_{i=1}^N \sigma_i(X)^q}$, for $q=1/2$ and $2/3$, can be achieved using a convex penalty on the factors $B$ and $C$. A generalization to general values of $q$ is given in \cite{xu-etal-AAAI-2017}.
While this reduces shrinking bias to some extent, it results in a non-differentiable and non-convex formulation that is optimized with ADMM.

It is important to note that many of the above methods that are considered state-of-the-art have been developed for low level vision tasks such as image denoising, inpainting, alignment and background subtraction. The ground truth for these models are often of higher rank
than models in~\eg{}~structure from motion, making it possible to obtain good results with weaker regularization. Additionally, as we will see in the experiments, more difficult data terms prevent rapid convergence of the splitting methods they often employ. 

\section{Non-Convex Penalties and Shrinking Bias}
In this section we will show how to formulate regularization terms of the type

\begin{equation}
\reg (X) = \sum_{i=1}^N f(\sigma_i(X)),
\label{eq:regdef}
\end{equation}
by penalizing the factors of the factorization~$X=BC^T$.
We assume that $B$ and $C$ have $k$ columns, making $\sigma_i(X)=0$ if $i>k$ and $\rank(X)\leq k$. Note, however, that we are aiming to achieve a lower rank using the regularization term. In many applications, the sought rank is unknown and should be determined by the regularization. We therefore set $k$ large enough not to exclude the optimal solution. As we shall see in Section~\ref{sec:optloc}, this ability to over-parameterize can be used to ensure optimality.

\begin{theorem}\label{thm:main}
If $f$ is concave, non-decreasing on $[0,\infty)$ and $f(0)=0$ then 
\begin{equation}
\reg(X) = \min_{BC^T = X} \sum_{i=1}^k f(\|B_i\| \|C_i\|), 
\label{eq:bilinear1}
\end{equation}
where $B_i$ and $C_i$, $i=1,...,k$ are the columns of $B$ and $C$ respectively.
\end{theorem}
\begin{proof}
The result is a consequence of the fact that $\reg$ will fulfill a triangle inequality $\reg(X+Y)\leq \reg(X)+\reg(Y)$ under the assumptions on $f$. This is clear from Theorem~4.4 in~\cite{uchiyama-2005} which shows that 
\begin{equation}
\sum_{i=1}^N f(\sigma_i(X+Y)) \leq \sum_{i=1}^N (f(\sigma_i(X))+\sum_{i=1}^N f(\sigma_i(Y))). 
\end{equation}
Applying this to $X = BC^T = \sum_{i=1}^k B_i C_i^T$ we see that
\begin{equation}
\reg(X) = \reg(\sum_{i=1}^k B_i C_i^T) \leq \sum_{i=1}^k \reg(B_i C_i^T).
\end{equation}
Since $\rank(B_i C_i^T)=1$ we also have
\begin{equation}
\reg(B_i C_i^T) = f(\sigma_1(B_i C_i^T)) = f(\|B_i C_i^T\|_F). 
\end{equation}
Lastly, since $\|B_i C_i^T\|_F = \|B_i\| \|C_i\|$ we get
\begin{equation}
\reg(X) \leq\sum_{i=1}^k f(\|B_i\| \|C_i\|). 
\end{equation}
To see that equality can be achieved, let \mbox{$B_i = \sqrt{\sigma_i(X)}U_i$} and $C_i = \sqrt{\sigma_i(X)}V_i$, where $X = \sum_{i=1}^k \sigma_i(X) U_i V_i^T$ is the SVD of $X$.
Then, $BC^T=X$ and $f(\|B_i\|\|C_i\|) = f(\sigma_i(X))$.
\end{proof}

While the above result allows optimization over the factors $B$ and $C$ we note that it yields an objective that is non-differentiable at $\|B_i\|\|C_i\| = 0$.
Next we reformulate the objective to achieve a differentiable problem formulation.

\begin{cor}\label{cor:main} Under the assumptions of Theorem~\ref{thm:main}, it follows that
$\reg(X) = \min_{X=BC^T} \tilde{\reg}(B,C)$, where
\begin{equation}
\tilde{\reg}(B,C) = \sum_{i=1}^k f\left(\frac{\|B_i\|^2+\|C_i\|^2}{2}\right).
\label{eq:Rtildedef}
\end{equation}
If $f$ is differentiable then $\tilde{R}(B,C)$ is also differentiable.
\end{cor}
\begin{proof}
By the rule of arithmetic and geometric means 
\begin{equation}
\|B_i\|\|C_i\| \leq \frac{1}{2} (\|B_i\|^2+\|C_i\|^2),
\end{equation}
with equality if $\|B_i\| = \|C_i\|$ which is achieved when \mbox{$B_i = \sqrt{\sigma_i(X)}U_i$} and $C_i = \sqrt{\sigma_i(X)}V_i$.
Since $f$ is assumed to be non-decreasing,
it follows from~\eqref{eq:bilinear1}, that
$\reg(X) = \min_{X=BC^T} \tilde{\reg}(B,C)$.
The differentiability of $\tilde{\reg}(B,C)$ is now trivially checked using the chain rule.
\end{proof}

\newcommand{\STAB}[1]{\begin{tabular}{@{}c@{}}#1\end{tabular}} 

\begin{table*}[htpb]
\centering
\caption{Distance to ground truth (normalized) mean valued over 20 problem instances for
different percentages of missing data, missing data patterns and noise levels~$\sigma$. Best results are marked in bold.}
\setlength\tabcolsep{0.1cm}
{\footnotesize
\rowcolors{2}{blue!15}{white!10}
\begin{tabular}{r | rrrrrr | rr | rr | r}

Missing\\data (\%) & PCP~\cite{candes-etal-acm-2011} & WNNM~\cite{gu-2016} &
Unifying~\cite{cabral-etal-iccv-2013} & LpSq~\cite{nie-2012} &
S12L12~\cite{shang-etal-2018} & S23L23~\cite{shang-etal-2018} &
IRNN~\cite{canyi2015} & APGL~\cite{toh-yun-2010} &
$\norm{\cdot}_*$~\cite{boyd-etal-2011}  & $\reg$~\cite{larsson-olsson-ijcv-2016} & Our\\

\toprule 
0 & \textbf{0.0000} & \textbf{0.0000} & \textbf{0.0000} & \textbf{0.0000} & 0.0002 & 0.0002 & \textbf{0.0000} & \textbf{0.0000} & 0.1727 & \textbf{0.0000} & \textbf{0.0000} \\
10 &0.0885 & 0.0028 & 0.0713 & 0.0213 & 0.0309 & 0.0071 & \textbf{0.0000} & \textbf{0.0000} & 0.1998 & \textbf{0.0000} & \textbf{0.0000} \\
20 &0.2720 & 0.2220 & 0.1491 & 0.0170 & 0.0412 & 0.0209 & \textbf{0.0000} & \textbf{0.0000} & 0.2223 & 0.0128 & \textbf{0.0000} \\
30 &0.7404 & 0.4787 & 0.7499 & 0.0003 & 0.0818 & 0.0895 & \textbf{0.0000} & 0.0014 & 0.2897 & 0.2346 & \textbf{0.0000} \\
40 &1.0000 & 0.6097 & 0.9553 & 0.1083 & 0.1666 & 0.1360 & \textbf{0.0000} & 0.0017 & 0.3374 & 0.2198 & \textbf{0.0000} \\
\multirow{-6}{*}{\STAB{\rotatebox[origin=c]{90}{\scriptsize Uniform ($\sigma=0.0$)}}\phantom{AA}}
50 &1.0000 & 0.7170 & 1.0000 & 0.0315 & 0.1376 & 0.1001 & 0.0003 & 0.0301 & 0.4266 & 0.2930 & \textbf{0.0000}
\\

%

\midrule 

0 & \textbf{0.0000} & \textbf{0.0000} & \textbf{0.0000} & \textbf{0.0000} & 0.0002 & 0.0002 & \textbf{0.0000} & \textbf{0.0000} & 0.1810 & \textbf{0.0000} & \textbf{0.0000} \\
10 & 0.3160 & 0.2734 & 0.1534 & 0.0839 & 0.1296 & 0.1233 & 0.0772 & 0.0834 & 0.2193 & 0.0793 & \textbf{0.0658} \\
20 & 0.4877 & 0.4499 & 0.3017 & 0.1650 & 0.2389 & 0.2456 & \textbf{0.1010} & 0.1786 & 0.3436 & 0.2494 & 0.1018 \\
30 & 0.5821 & 0.5395 & 0.5486 & 0.2520 & 0.3289 & 0.3160 & \textbf{0.1189} & 0.2572 & 0.4299 & 0.3421 & \textbf{0.1189} \\
40 & 0.7072 & 0.6317 & 0.7376 & 0.2853 & 0.4084 & 0.4110 & 0.1417 & 0.2913 & 0.4825 & 0.5004 & \textbf{0.1385} \\
\multirow{-6}{*}{\STAB{\rotatebox[origin=c]{90}{\scriptsize Tracking ($\sigma=0.0$)}}\phantom{AA}}
50 & 0.8125 & 0.7257 & 0.9521 & 0.4178 & 0.4267 & 0.4335 & 0.2466 & 0.4047 & 0.5754 & 0.6503 & \textbf{0.2214} \\

\midrule 

0 & 0.0409 & 0.0207 & 0.0407 & 0.0450 & 0.0437 & 0.0435 & 0.0448 & 0.0191 & 0.1581 & \textbf{0.0166} & \textbf{0.0166}  \\
10 &0.3157 & 0.2734 & 0.1585 & 0.0848 & 0.0529 & 0.0518 & 0.0625 & 0.0696 & 0.2312 & 0.0488 & \textbf{0.0438}  \\
20 &0.4771 & 0.4338 & 0.3480 & 0.1394 & 0.0995 & \textbf{0.0982} & 0.1090 & 0.1188 & 0.3109 & 0.2071 & 0.0983  \\
30 &0.5801 & 0.5225 & 0.4726 & 0.2026 & 0.2468 & 0.2592 & 0.1646 & 0.1993 & 0.3820 & 0.3465 & \textbf{0.1475}  \\
40 &0.7122 & 0.6148 & 0.8638 & 0.2225 & 0.3292 & 0.3252 & 0.1357 & 0.2110 & 0.4800 & 0.4599 & \textbf{0.1273}  \\
\multirow{-6}{*}{\STAB{\rotatebox[origin=c]{90}{\scriptsize Tracking ($\sigma=0.1$)}}\phantom{AA}}
50 &0.7591 & 0.6819 & 0.9216 & 0.4105 & 0.4883 & 0.4811 & 0.3342 & 0.3639 & 0.5652 & 0.5930 & \textbf{0.3329}  \\

\bottomrule
\end{tabular}
\label{tab:synth}
}
\end{table*}

We are particularly interested in the case \eqref{eq:fmu}
since, with this choice, it is known that the global minimizer of \eqref{eq:generalregformulation} is the same as that of $\mu \rank(X) +\|\A X - b\|^2$ if $\|\A\| < 1$, see \cite{carlsson2016convexification} for a proof. 
Note that $f_\mu$ is a special case of the MCP class \cite{zhang2010nearly}.
With this choice $\tilde{\reg}(B,C)$ is differentiable and the second derivatives are also defined almost everywhere except in the transition $\frac{\|B_i\|^2+\|C_i\|^2}{2} = \sqrt{\mu}$ where the function switches from quadratic to constant.

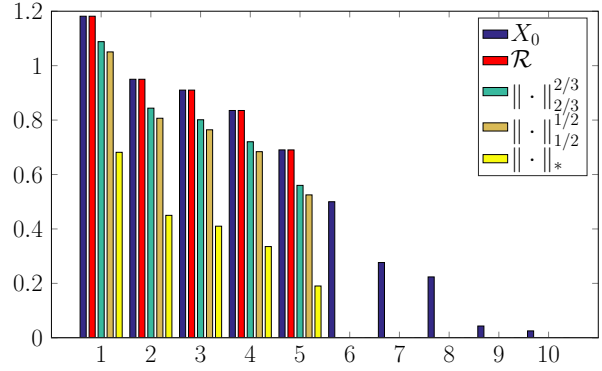
\begin{figure}[htb]
\centering
\resizebox{!}{50mm}{
%
%
\definecolor{mycolor1}{rgb}{0.20810,0.16630,0.52920}%
\definecolor{mycolor2}{rgb}{0.21783,0.72504,0.61926}%
\definecolor{mycolor3}{rgb}{0.85066,0.72990,0.33603}%
\definecolor{mycolor4}{rgb}{0.97630,0.98310,0.05380}%
\begin{tikzpicture}
\LARGE
\begin{axis}[%
width=5.25in,
height=3.138in,
at={(0.408in,0.306in)},
scale only axis,
bar shift auto,
xmin=0,
xmax=11,
xtick={ 1,  2,  3,  4,  5,  6,  7,  8,  9, 10},
ymin=0,
ymax=1.2,
axis background/.style={fill=white},
legend style={legend cell align=left, align=left, draw=white!15!black}
]
\addplot[ybar, bar width=0.123, fill=mycolor1, draw=black, area legend] table[row sep=crcr] {%
1	1.18161821420815\\
2	0.94997555258641\\
3	0.910115362069281\\
4	0.835249882476506\\
5	0.690395320608537\\
6	0.5\\
7	0.276619801327966\\
8	0.223747934287232\\
9	0.0433753911802577\\
10	0.0254059860675901\\
};
\addplot[forget plot, color=white!15!black] table[row sep=crcr] {%
0	0\\
11	0\\
};
\addlegendentry{$X_0$}

\addplot[ybar, bar width=0.123, fill=red, draw=black, area legend] table[row sep=crcr] {%
1	1.18161821420815\\
2	0.94997555258641\\
3	0.910115362069281\\
4	0.835249882476506\\
5	0.690395320608537\\
6	0\\
7	0\\
8	0\\
9	0\\
10	0\\
};
\addplot[forget plot, color=white!15!black] table[row sep=crcr] {%
0	0\\
11	0\\
};
\addlegendentry{$\reg$}

\addplot[ybar, bar width=0.123, fill=mycolor2, draw=black, area legend] table[row sep=crcr] {%
1	1.08815083489349\\
2	0.843836433346921\\
3	0.801187889278495\\
4	0.720374936962404\\
5	0.560119346571341\\
6	0\\
7	0\\
8	0\\
9	0\\
10	0\\
};
\addplot[forget plot, color=white!15!black] table[row sep=crcr] {%
0	0\\
11	0\\
};
\addlegendentry{$\|\cdot\|_{2/3}^{2/3}$}

\addplot[ybar, bar width=0.123, fill=mycolor3, draw=black, area legend] table[row sep=crcr] {%
1	1.05045472634362\\
2	0.806748163633621\\
3	0.764282966576485\\
4	0.683915382564861\\
5	0.525129406405147\\
6	0\\
7	0\\
8	0\\
9	0\\
10	0\\
};
\addplot[forget plot, color=white!15!black] table[row sep=crcr] {%
0	0\\
11	0\\
};
\addlegendentry{$\|\cdot\|_{1/2}^{1/2}$}

\addplot[ybar, bar width=0.123, fill=mycolor4, draw=black, area legend] table[row sep=crcr] {%
1	0.681618214208152\\
2	0.44997555258641\\
3	0.410115362069281\\
4	0.335249882476506\\
5	0.190395320608537\\
6	0\\
7	0\\
8	0\\
9	0\\
10	0\\
};
\addplot[forget plot, color=white!15!black] table[row sep=crcr] {%
0	0\\
11	0\\
};
\addlegendentry{$\|\cdot\|_*$}

\end{axis}
\end{tikzpicture}%
}
\caption{Singular values obtained when minimizing $\|X-X_0\|_F^2$ with the four regularizers $\reg(X)$ with $f=f_\mu$, $\|X\|_{1/2}^{1/2}$, $\|X\|_{2/3}^{2/3}$ and $\|X\|_*$. Large singular values are left unchanged by $\reg$.}
\label{fig:bias}
\end{figure}
We conclude this section by comparing the shrinking bias of our approach and three others that can also be optimized over the factorization. 
Theorem~\ref{thm:main} makes it possible to compute the global optimizer of $\tilde{\reg}(B,C)+\|BC^T-X_0\|_F^2$ since the equivalent problem
$\reg(X)+\|X-X_0\|_F^2$ has closed form solution in the $X$-parameterization. It is shown in~\cite{larsson-olsson-ijcv-2016} that with $f=f_\mu$ the solution is obtained by thresholding the singular values at $\sqrt{\mu}$. 
Similarly, closed form solutions are also available when regularizing $\|X-X_0\|_F^2$ with $\|\cdot\|_{1/2}$, $\|\cdot\|_{2/3}$ and  $\|\cdot\|_*$ \cite{shang-etal-2018}.
In Figure~\ref{fig:bias} we show the singular values obtained when regularizing $\|X-X_0\|_F^2$ with these four options, and for comparison the singular values of $X_0$.
For all methods we have selected regularization weights as small as possible so that the five smallest singular values are completely suppressed, which minimizes the bias.
While all choices, except $\reg$, subtract a part from the singular values that should be retained, the Schatten norms reduce the bias significantly compared to the nuclear norm. For the Schatten norms the bias is larger for singular values that are close to the threshold since 
the derivative of $\sigma^q$, $0<q<1$, decreases with increasing $\sigma$. For problem instances where there is a clear separation in size between singular values that should be retained and those that should be suppressed, it is likely that this can be done with negligible bias. Since $f'_\mu(\sigma)=0$ when $\sigma \geq \sqrt{\mu}$ this method does not affect the first five singular values.

\section{Overparameterization and Optimality}\label{sec:optloc}
The results of the previous section show that a global
optimizer $(B,C)$ of \eqref{eq:generalbilin} 
gives a solution $BC^T$ which is globally optimal in~\eqref{eq:generalregformulation}.
On the other hand, optimizing \eqref{eq:generalbilin} over $B$ and $C$ introduces additional stationary points, due to the non-linear parameterization, that are not present in \eqref{eq:generalregformulation}. One such point is $(B,C)=(0,0)$ where the gradients of $\|\A BC^T - b\|^2$ with respect to $B$ and $C$ vanish (in contrast to the gradient \wrt{}~$X$).
In this section we show that by overparametrizing, in the sense that we use $B$ and $C$ with more columns than the rank of the solution we seek, it is still possible to use properties of \eqref{eq:generalregformulation} to show optimality in \eqref{eq:generalbilin}.
We will exclusively use $f_\mu$ from \eqref{eq:fmu},
assume that $B$ and $C$ have $2k$ columns and study locally optimal solutions with $\rank(BC^T)<k$. The size of $B$ and $C$ makes it possible to parametrize line segments between such points and utilize convexity properties, see proof of Theorem~\ref{thm:lowrank-opt}. 
The following result (which is proven in Appendix~\ref{sec:proofs}) gives conditions that ensure that local minimality in \eqref{eq:generalbilin} implies that  \eqref{eq:generalregformulation} grows in all ``low rank'' directions.
\begin{theorem}\label{thm:dirderiv}
	Assume that~$(\bar{B},\bar{C})\in \mathbb{R}^{m\times 2k} \times \mathbb{R}^{n\times 2k} $,
where $\bar{B}=U\sqrt{\Sigma}$ and $\bar{C}=V\sqrt{\Sigma}$, and $\bar{X}=U\Sigma V^{\T}$,
is a local minimizer of \eqref{eq:generalbilin}
	with $\rank(\bar{X})<k$ and let
	$\N(X) = \reg(X)+\|\A X- b\|^2.$
	Then $\reg(\bar{X}) = \tilde{\reg}(\bar{B},\bar{C})$ and the directional derivatives $\N'_{\Delta X}(\bar{X})$, where $\Delta X = \tilde{X}-\bar{X}$ and $\rank(\tilde{X}) \leq k$, are non-negative.
\end{theorem}
Note that there can be local minimizers for which $\tilde{\reg}(\bar{B},\bar{C}) > \reg(\bar{B}\bar{C}^T)$ since $\tilde{\reg}$ is non-convex. 
From an algorithmic point of view we can, however, escape such points by taking the current iterate and recompute the factorization of $\bar{B}\bar{C}^T$ using SVD. If the SVD of $\bar{B}\bar{C}^T= \sum_{i=1}^r \sigma_i U_i V_i^T$ we update $\bar{B}$ and $\bar{C}$ to $\bar{B}_i = \sqrt{\sigma_i}U_i$ and $\bar{C}_i = \sqrt{\sigma_i}V_i$, which we know reduces the energy and gives $\tilde{\reg}(\bar{B},\bar{C}) = \reg(\bar{B}\bar{C}^T)$.


\begin{figure*}[t!]
\centering
\includegraphics[width=0.495\textwidth]{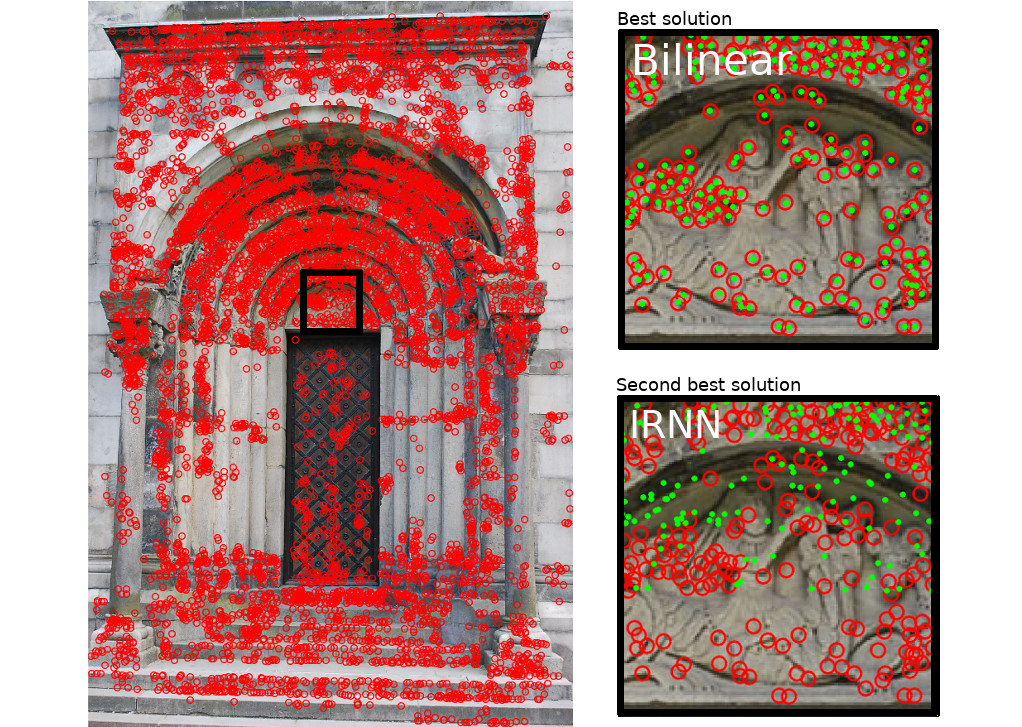}%
\includegraphics[width=0.495\textwidth]{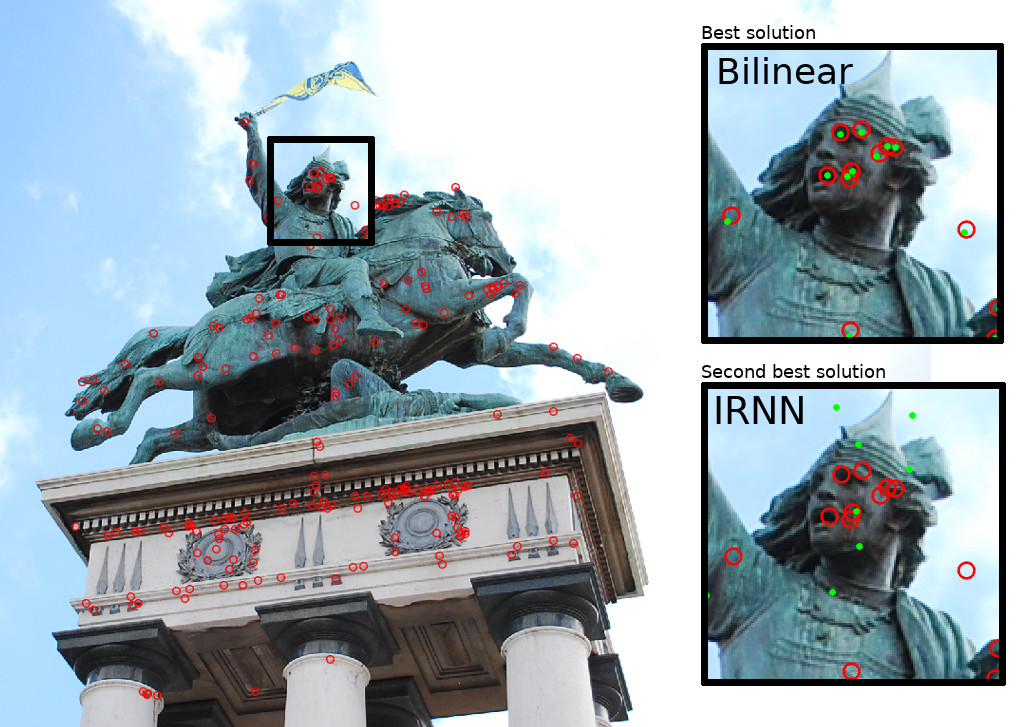}\\
\includegraphics[width=0.95\textwidth]{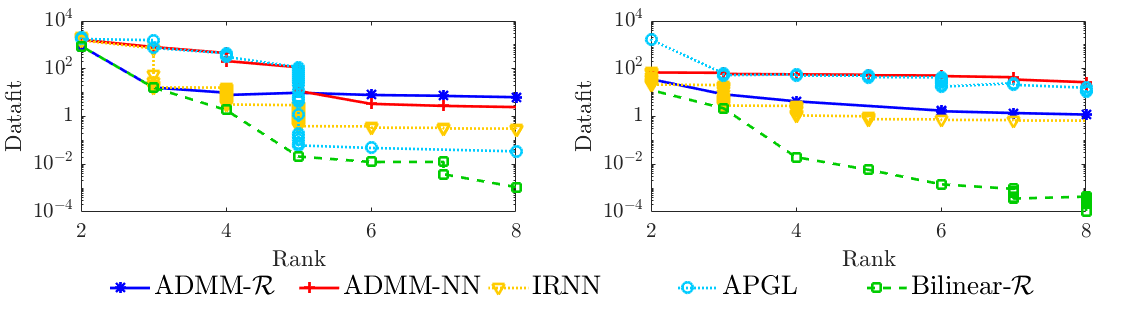}%
\caption{Comparison of reprojection error obtained using the bilinear formulation
and ADMM, for datasets \emph{Door} and \emph{Vercingetorix}~\cite{olsson-engqvist-scia-2011}. The red circles mark the
feature points and the green dots the projected image points obtained from the different methods. The best rank~4 solution for the respective method was used. The control parameter $\eta=0.5$ in both experiments.}
\label{fig:pose_reproj}
\end{figure*}

Theorem~\ref{thm:dirderiv} allows us to derive optimality conditions using the properties of \eqref{eq:generalregformulation}.
As a simple example, consider the case where $\|\A X\|^2 \geq \|X\|^2$, which makes \eqref{eq:generalregformulation} convex \cite{carlsson2016convexification}, and let $B$ and $C$ have $2k$ columns. Suppose that we find a local minimizer $(\bar{B},\bar{C})$ fulfilling the assumptions of Theorem~\ref{thm:dirderiv}. Then the derivative along a line segment towards any other low rank matrix is non-decreasing, and therefore $\bar{B}\bar{C}^T$ is the global optimum of \eqref{eq:generalregformulation} over the set of matrices with $\rank \leq k$ by convexity.

Below we give a result that goes beyond convexity and applies to the important class \cite{recht-etal-siam-2010} of problems that obey 
the RIP constraint \eqref{eq:RIP}. Let $\A^*$ denote the adjoint operator of~$\A$, then:
\begin{theorem}\label{thm:lowrank-opt}
	Assume that $(\bar{B},\bar{C})$ is a local minimizer of \eqref{eq:generalbilin},
fulfilling the assumptions of Theorem~\ref{thm:dirderiv}. If the singular values of $Z = (I-\A^* \A)\bar{B}\bar{C}^T+\A^*b$ fulfill $\sigma_i(Z) \notin [(1-\delta_{2k})\sqrt{\mu}, \frac{\sqrt{\mu}}{(1-\delta_{2k})}]$
	then $\bar{B} \bar{C}^T$ is the solution of \eqref{eq:regminprobl} and \eqref{eq:rankminprobl}.
\end{theorem}
The proof builds on the results of \cite{olsson-etal-iccv-2017} and is given in Appendix~\ref{sec:proofs}.
The assumption that the singular values of $Z$ are not too close to the threshold $\sqrt{\mu}$ 
is a natural restriction which is valid when the noise level is not too large.
In case of exact data, \ie{}~$b = \A X_0$, where $\rank(X_0) = r$ it is trivially fulfilled for any choice of $\mu$ such that $\sqrt{\mu}< (1-\delta_{2k})\sigma_r(X_0)$
since we have $Z=X_0$. For additional details on $Z$'s dependence on noise see \cite{carlsson2018unbiased}.

The above result is similar in spirit to those of \cite{recht-etal-siam-2010,haeffele-vidal-arxiv-2017}, which show that, in the convex case,
having $2k$ columns and rank $2k-1$ is enough to ensure that a local minimizer is global.
For the proof in our non-convex case we need rank at most $k-1$. Presently, it is not clear if our assumption can be relaxed to match that of the convex case or not.

\section{An Iterative Reweighted VarPro Algorithm}\label{sec:implement}

In this section we give a brief overview of our algorithm for minimizing \eqref{eq:generalbilin}. A more detailed description is given in Appendix~\ref{sec:implementationdetails}.

Given a current iterate, $B^{(t)}$ and $C^{(t)}$, the first step of our algorithm is to replace the term $\tilde{\reg}(B,C)$ with a quadratic function. To do this we note that by the Taylor expansion
$
f(x) \approx f(x_0)+f'(x_0)(x-x_0),
$
minimizing $f(x)$ and $f'(x_0)x$ around $x_0$ is roughly the same (ignoring constants). Inserting $x_0 = \frac{\|B^{(t)}_i\|^2+\|C^{(t)}_i\|^2}{2}$ and
$x = \frac{\|B_i\|^2+\|C_i\|^2}{2}$ now gives our approximation
\begin{equation}
\sum_{i=1}^k w^{(t)}_i(\|B_i\|^2+\|C_i\|^2)+\|\A B C^T - b\|^2,
\label{eq:weighedapprox}
\end{equation}
where $w^{(t)}_i=\fr{1}{2}f'\!\left((\|B^{(t)}_i\|^2+\|C^{(t)}_i\|^2)/2\right)$. Here $B_i^{(t)}$ and $C_i^{(t)}$ are the $i$:th columns of $B^{(t)}$ and $C^{(t)}$, respectively. 
Minimizing \eqref{eq:weighedapprox} over $C$ is now a least squares problem with closed form solution. Inserting this solution into the original problem gives a nonlinear problem in $B$ alone,
which is what VarPro solves.
We use the so called Ruhe and Wedin (RW2) approximation with a dampening term $\lambda \|B-B^{(t)}\|_F^2$,    
see~\cite{hong-etal-cvpr-2017} for details. In each step of the VarPro algorithm we update the weights $w_i^{(t)}$.

As previously mentioned, there can be stationary points for which $\tilde{\reg}(B,C) > \reg(B C^T)$.
In each iteration we therefore take the current iterate and recompute the factorization of $B^{(t)}C^{(t)T}$ using SVD. If the SVD of $B^{(t)}C^{(t)T}= \sum_{i=1}^r \sigma_i U_i V_i^T$ we update $B^{(t)}$ and $C^{(t)}$ to $B^{(t)}_i = \sqrt{\sigma_i}U_i$ and $C^{(t)}_i = \sqrt{\sigma_i}V_i$ which we know reduces the energy and gives $\tilde{\reg}(B^{(t)},C^{(t)}) = \reg(B^{(t)}C^{(t)T})$.

Our approach can be seen as iteratively reweighted nuclear norm minimization~\cite{canyi2015}; however, our bilinear formulation
allows us to use quadratic approximation, thus benefiting from second order convergence in the neighborhood of a local minimum.

\section{Experiments}
In this section we will show the versatility and strength of the proposed method, focusing on computer vision
problems.
In Section~\ref{sec:pOSE} we show an example where state-of-the-art methods fail to
achieve a value close to global optimality. We include two more examples of real problems,
in Appendix~\ref{sec:moreexp}: background extraction and photometric stereo. In both cases
our method shows superior performance.
In the main paper we focus on the trade-off between datafit and rank, but show, in the
examples in the supplementray material, the added benefits of convergence speed using the
proposed method. This is done by minimizing the same energy with ADMM and the proposed method,
in which case the splitting schemes can be tediously slow.
In all experiments our proposed method is initialized randomly,
with zero mean and unit variance.

\subsection{Synthetic Missing Data Problem}\label{sec:synth}
Let~$\hada$ denote the Hadamard product, and consider the
missing data formulation
\begin{equation}
   \min_{\mat{X}}\mu\rank(\mat{X}) + \norm{\mat{W}\hada(\mat{X}-\mat{M})}_F^2,
   \label{eq:missingdata2}
\end{equation}
where~$\mat{M}$ is a measurement matrix and~$\mat{W}$ a missing data mask with
entries $w_{ij}=1$ if the entry is known, and zero otherwise.

In low-level vision applications such as denoising and image inpainting a uniformly random
missing data pattern is often a reasonable approximation of the distribution; however, for
structure from motion, the missing data pattern is often highly structured. To this end,
we investigate two kinds of patterns: uniformly random and ``tracking failure".
In order to construct realistic patterns of tracking failure, we use the method
in~\cite{larsson-olsson-cvpr-2017}. This is done
by randomly selecting if a track should have missing data (with uniform probability), then
select (with uniform probability, starting after the first few frames) in which
image tracking failure occurs. If a track is lost, it is not restarted.

\begin{figure}[h!]
\centering
\includegraphics[width=0.475\textwidth]{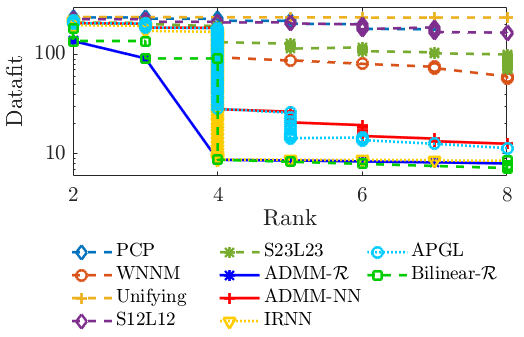}
\caption{Rank vs datafit for the synthetic experiment in Section~\ref{sec:synth}.
No true low rank solution using LpSq~\cite{nie-2012} could be found, regardless of
the choice of parameters.}
\label{fig:synth_rank_vs_datafit}
\end{figure}
We generate random ground truth matrices~\mbox{$\mat{M}_{0}\in\R^{32\times 512}$} of rank~4, which
can be expressed as $\mat{M}_{0}=\mat{U}\mat{V}^{\T}$, where
$\mat{U}\in\R^{32\times 4}$ and~$\mat{V}\in\R^{512\times 4}$. The entries of $\mat{U}$
and $\mat{V}$ are normal distributed with zero mean and unit variance. The measurement
matrix~$\mat{M}=\mat{M}_0+\mat{N}$, where $\mat{N}$ simulates noise and has
normal distributed entries
with zero mean and variance~$\sigma^2$.

\begin{figure*}[thb]
\centering
\setlength\tabcolsep{0.12cm}
\def\arraystretch{0.95}
\def\w{42.16mm}
\def\ww{20mm}

\newcommand{\vcenteredinclude}[1]{\begingroup
\setbox0=\hbox{\includegraphics[width=\ww]{#1}}%
\parbox{\wd0}{\box0}\endgroup}
\begin{tabular}{cccc}
\emph{Drink}\vcenteredinclude{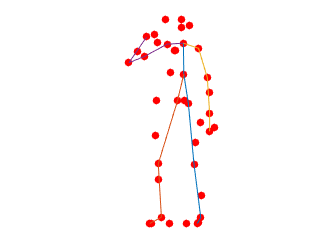} &
\emph{Pickup}\vcenteredinclude{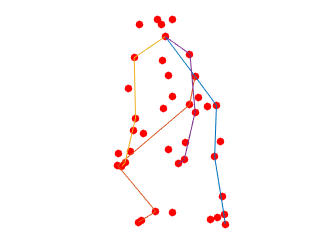} &
\emph{Stretch}\vcenteredinclude{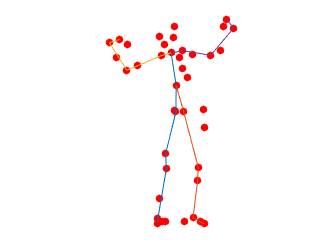} &
\emph{Yoga}\vcenteredinclude{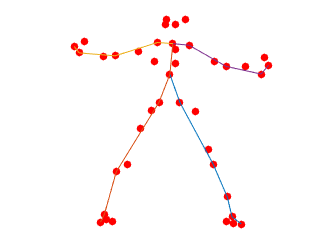}
\vspace{-0.15cm} \\
\includegraphics[width=\w]{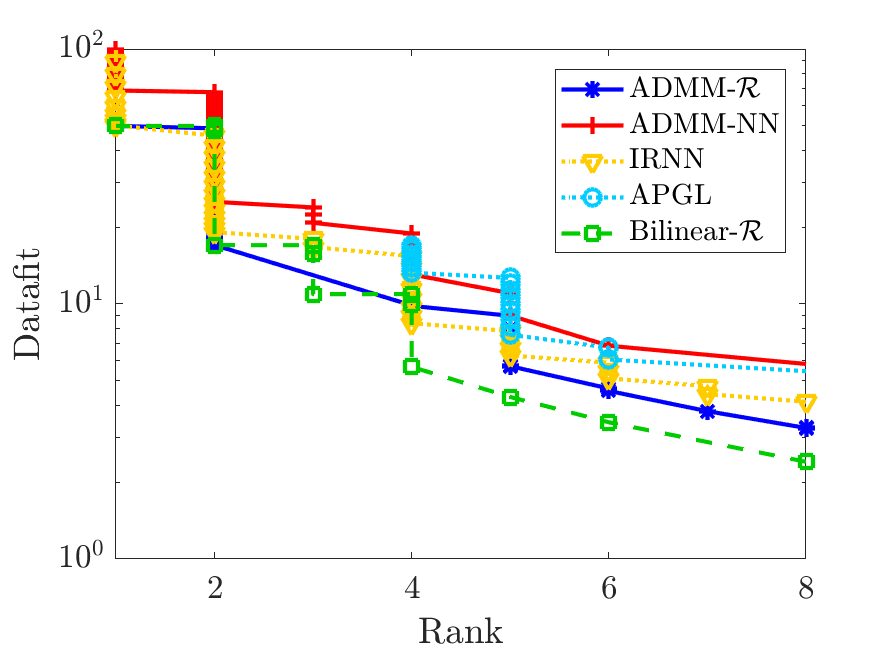} &
\includegraphics[width=\w]{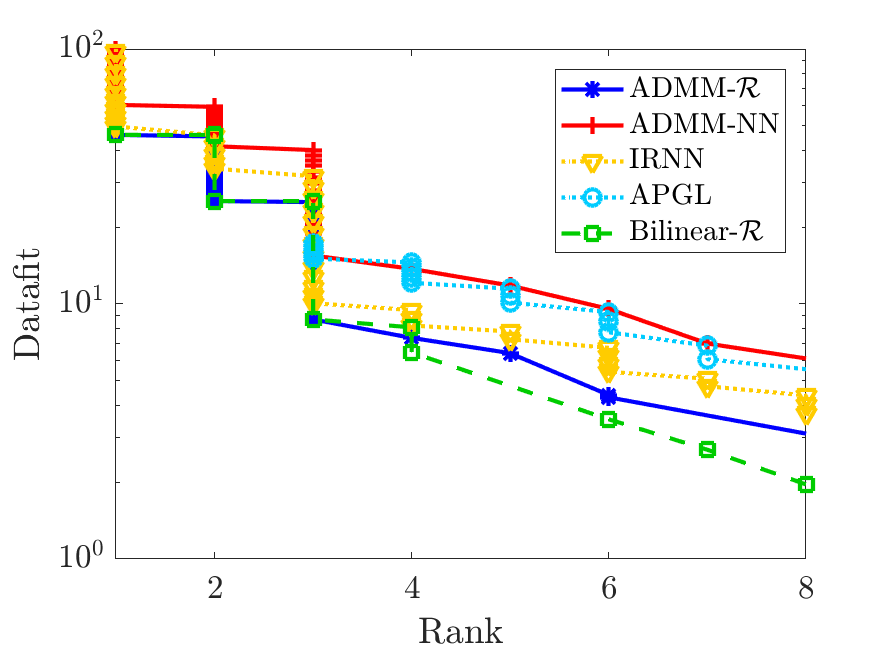} &
\includegraphics[width=\w]{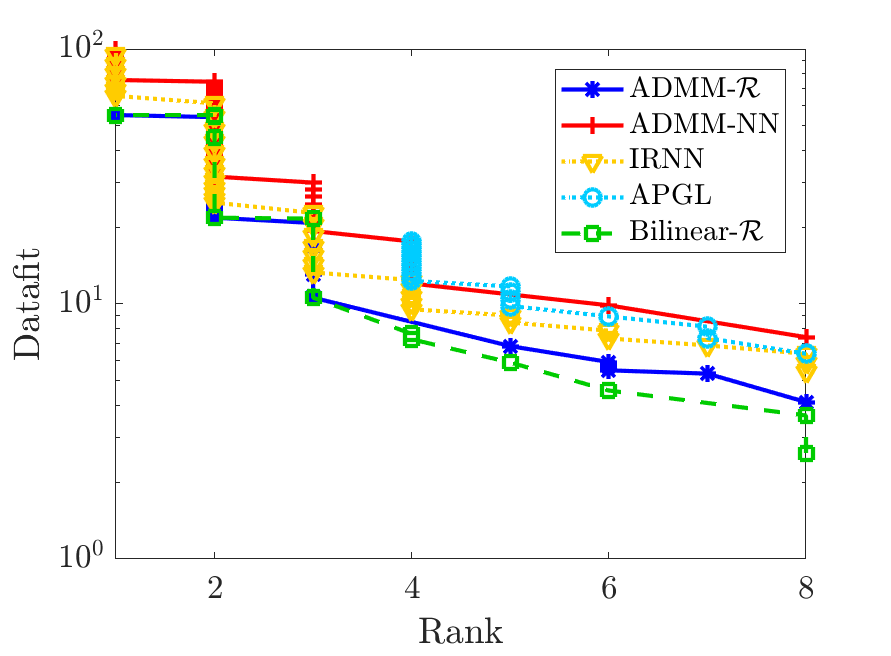} &
\includegraphics[width=\w]{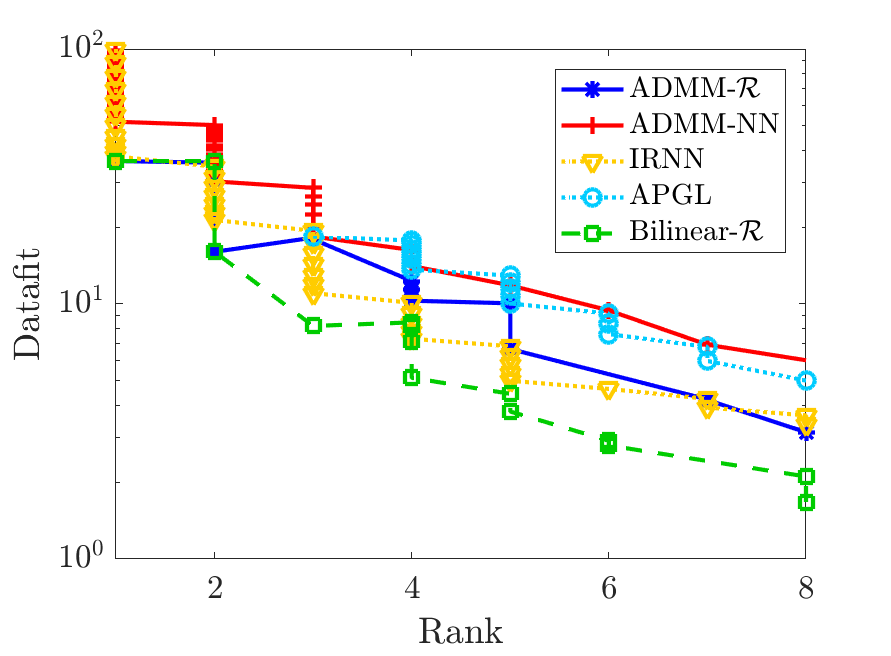} \\
\end{tabular}
\caption{\emph{Top row:} Example frames from the MOCAP dataset of the \emph{drink}, \emph{pickup},
\emph{stretch} and \emph{yoga} sequences. \emph{Last row:} The bilinear method finds the same or a better datafit compared to the other methods for all ranks.}
\label{fig:MOCAP}
\vspace{-0.4cm}
\end{figure*}

Our proposed method is compared to a variety of different methods~\cite{cabral-etal-iccv-2013,candes-etal-acm-2011,gu-2016,nie-2012,shang-etal-2018,canyi2015,toh-yun-2010,boyd-etal-2011,larsson-olsson-ijcv-2016}.
For the methods that need an initial estimate of the rank as input, the rank estimation
heuristic by Shang~\etal{}~\cite{shang-etal-2018} is used.
The regularization parameter is set to
$\lambda=\sqrt{\max(m,n)}$, given a sought $m\times n$ matrix, as proposed by~\cite{candes-etal-acm-2011,shang-etal-2018}. In
case other parameters should be provided, the one recommended from the respective authors have been
used.
The number of columns, for our proposed method, is set to
$k=8$, \ie{}~twice the rank of the original matrix~$M_0$.
We exclusively use the~$f_\mu$ regularization~\eqref{eq:fmu},
and use $\sqrt{\mu}=\lambda$. Since $f_\mu$ is a special case of MCP, it is used for IRNN
as well.
Furthermore, we include the results for regularizing with nuclear norm~\cite{boyd-etal-2011} and
$f_\mu$~\eqref{eq:fmu} using ADMM, as
proposed in~\cite{larsson-olsson-ijcv-2016}. Note that ADMM comes without optimality guarantees,
however, it has been shown to work well for several computer vision problems in
practice~\cite{larsson-olsson-ijcv-2016,olsson-etal-iccv-2017}.
Several of the compared methods solve the robust PCA problem, thus
also include a sparse component, which is not taken into account.

The results are shown in Table~\ref{tab:synth}. Note that most algorithms perform significantly
better for the uniformly random missing data pattern, than compared to the structured missing
data pattern. Our proposed method outperforms all other methods in this comparison.

Since the final rank of the estimated matrix is not necessarily the same as that of~$M_0$, we show the rank vs datafit obtained when varying the regularization parameter~$\lambda$ in Figure~\ref{fig:synth_rank_vs_datafit}.
It is evident from the results that the only candidates that yield an acceptable result for low rank solutions are ADMM with $f_\mu$, IRNN with MCP and our proposed method.

\subsection{pOSE: Pseudo Object Space Error}
\label{sec:pOSE}
The Pseudo Object Space Error (pOSE) objective
combines affine and projective camera models

\begin{align}
    \ell_{\textsf{OSE}} &= \sum_{(i,j)\in\Omega} \norm{(\mat{P}_{i,1:2}\tilde{\vec{x}}_j-(\vec{p}^{\T}_{i,3}\tilde{\vec{x}}_j)\vec{m}_{i,j})  }^2, \\
    \ell_{\textsf{Affine}} &= \sum_{(i,j)\in\Omega} \norm{\mat{P}_{i,1:2}\tilde{\vec{x}}_j-\vec{m}_{i,j}}^2, \\
    \ell_{\textsf{pOSE}} &= (1-\eta)\ell_{\textsf{OSE}}+\eta\ell_{\textsf{Affine}},
\end{align}
where~$\ell_{\textsf{OSE}}$ is the object space error and~$\ell_{\textsf{Affine}}$ is the affine projection error.
Here $\mat{P}_{i,1:2}$ denotes the first two rows, $\vec{p}_{i,3}$ the third row of the $i$:th camera matrix, and
$\tilde{\vec{x}}_j$ is the $j$:th 3D point in homogeneous coordinates.
The control parameter~$\eta\in[0,1]$ determines the impact of the respective camera model.
This objective was introduced in~\cite{hong-zach-cvpr-2018} to be used in
a first stage of an initialization-free bundle adjustment pipeline,
optimized using VarPro.

The~$\ell_{\textsf{pOSE}}$ objective is linear, and acts on low-rank components
$\mat{P}$ and~$\mat{X}$, which are constrained by \mbox{$\rank(PX^{\T})=4$}.
Instead of enforcing the rank constraint, we replace it as before with a relaxation.
By not enforcing the rank constraint we demonstrate the ability of the methods to make
accurate trade-offs between minimizing the rank and fitting the data.
Since the objective now becomes more complex, and is no longer compatible with the
missing data formulations, only IRNN and APGL are directly applicable, as well as the
ADMM approach using $f_\mu$ and nuclear norm.
We use two real-life datasets with various amounts of camera locations and 3D points:
\emph{Door} with 12 images, resulting in seeking a matrix of size
$36\times 8850$ and \emph{Vercingetorix}~\cite{olsson-engqvist-scia-2011} with 69 images, resulting
in seeking a matrix of size~$207\times 1148$, both of which have rank~4.
\footnote{
The datasets are available here:
\url{http://www.maths.lth.se/matematiklth/personal/calle/dataset/dataset.html}.}

As in the synthetic experiment from Section~\ref{sec:synth}, the regularization parameter is
varied and the resulting rank and datafit is stored and reported in Figure~\ref{fig:pose_reproj}.
To visualize the results, we considered the best rank~4 approximations, and show the
reprojected points and the corresponding measured
points obtained from the best method (ours in both cases) and the second best (IRNN in both cases),
see Figure~\ref{fig:pose_reproj}.
As is readily seen by ocular inspection, the rank~4 solution obtained by our proposed method
significantly outperforms those of other state-of-the-art methods.

\subsection{Non-Rigid Structure From Motion}
In this section we test our approach on non-rigid reconstruction (NRSfM) with the CMU Motion Capture (MOCAP) dataset.
In NRSfM, the complexity of the deformations are controlled by some mild assumptions
of the object shapes. Bregler \etal{}~\cite{bregler-etal-cvpr-2000} suggested that the set of
all possible configurations of the objects are spanned by a low dimensional linear
basis of dimension~$K$. In this setting, the non-rigid shapes $X_i\in\R^{3\times n}$
can be represented as $X_i=\sum_{k=1}^Kc_{ik}B_k$, where $B_k\in\R^{3\times n}$ are
the basis shapes and $c_{ik}\in\R$
the shape coefficients. This way, the matrix~$X_i$ contains the world coordinates of point $i$, hence
the observed image points are given by $x_i=R_iX_i$. We will assume orthographic
cameras, \ie{} $R_i\in\R^{2\times 3}$ where~$R_iR_i^{\T}=I_2$.
As proposed by Dai \etal{}~\cite{dai-etal-ijcv-2014}, the problem can be turned into a
low-rank factorization problem by reshaping and stacking the non-rigid shapes~$X_i$.
Let $X_i^\sharp\in\R^{1\times 3n}$ denote the concatenation of the rows in~$X_i$,
and create $X^\sharp\in\R^{F\times 3n}$ by stacking~$X_i^\sharp$. This allows
us to decompose the matrix~$X^\sharp$ in the low-rank factors $X^\sharp=CB^\sharp$, where
$C\in\R^{F\times K}$ contains the shape coefficients $c_{ik}$ and $B^\sharp\in\R^{K\times 3n}$
is constructed as $X^\sharp$ and contains the basis elements.

A suitable objective function is thus given by
\begin{equation}\label{eq:mocap_objective}
    \mu\rank(X^\sharp)+\norm{RX-M}^2_F,
\end{equation}
where~$R\in\R^{2F\times 3F}$ is a block-diagonal matrix with the camera matrices $R_i$ on
the main diagonal, $X\in\R^{3F\times n}$ is the concatenation of the 3D points $X_i$, and
$M\in\R^{2F\times n}$ is the concatenated observed image points $x_i$.
By replacing
the rank penalty with a relaxation and minimize it using the proposed method and the
methods used in the previous section. The regularization parameter is varied for the
respective methods in order to obtain a rank 1--8 solution, and the respective datafit
is reported in Figure~\ref{fig:MOCAP}, for four different sequences.

In all sequences, the best datafit for each rank level is obtained by our proposed method. IRNN and
ADMM using~$f_\mu$ is able to give the same, or very similar, datafit for lower ranks, but for
solutions with rank larger than four our method consistently reports a lower value than the
competing state-of-the-art methods.

\section{Conclusions}
In this paper we presented a unification of bilinear parameterization and rank regularization.
Robust penalties for rank regularization has often been used together with splitting schemes, but it
has been shown that such methods yield unsatisfactory results for ill-posed problems in several
computer vision applications. By using the bilinear formulation, the objective functions
become differentiable, and convergence rates in the neighborhood of a local minimum are faster.
Furthermore, we showed that theoretical optimality results known from the regularization formulations
can be lifted to the bilinear formulation.

Lastly, the generality of the proposed framework allows for a wide range of problems,
some of which, have not been amenable by state-of-the-art methods, but
have been proven successful using our proposed method.

{\small
\bibliographystyle{ieee}

}

\newpage

\renewcommand{\thesection}{\Alph{section}}
\setcounter{section}{0}
\section{Proofs}\label{sec:proofs}

In this section we present the proofs of Theorems~\ref{thm:dirderiv} and~\ref{thm:lowrank-opt}.
Our analysis will make use of the differentiable objective 
\begin{equation}
\D(B,C) := \tilde{\reg}(B,C) + \| \A BC^T - b \|^2,
\end{equation} 
the non-convex function
\begin{equation}
\N(X) := \reg(X)+ \| \A X - b \|^2,
\end{equation}
and the convex function
\begin{equation}
\C(X) = \reg(X)+\|X-Z\|_F^2.
\end{equation}
We will also use the functions 
\begin{eqnarray}
& \tilde{G}(B,C) = \tilde{\reg}(B,C) + \|BC^T\|_F^2, \\
& G(X) = \reg(X) + \|X\|_F^2, \label{eq:Gdef} \\
& H(X) = \|\A X - b\|^2-\|X\|_F^2. \label{eq:Hdef}
\end{eqnarray}
Note that $\D(B,C) = \tilde{G}(B,C)+H(BC^T)$ and $\N(X) = G(X)+H(X)$.
Throughout the section we use $f = f_\mu$ with $f_\mu$ as in \eqref{eq:fmu} (of the main paper) 
but for simplicity of notation we will suppress the subscript $\mu$.
Furthermore, the subdifferential $\partial G(X)$ of $G$ will be of
importance. 
Let $g(x) = f(|x|) + x^2$. The scalar function $g$ has
\begin{equation}
\partial g(x) = \begin{cases}
2x & |x| \geq \sqrt{\mu} \\
2\sqrt{\mu}\text{sign}(x) & 0 < |x| \leq \sqrt{\mu}\\
2\sqrt{\mu}[-1,1] & x= 0
\end{cases}.
\end{equation}
The following lemma shows how to compute $\partial G$ for the matrix case using $\partial g$.
\begin{lemma}
	The subdifferential of $G(X)$ is given by
	\begin{equation}
	\begin{aligned}
	\partial G(X) &= \{U\partial g(\Sigma) V^T + M\;:\;\sigma_1(M) 
	\leq 2\sqrt{\mu},\\
	&\qquad\qquad\qquad U^TM = 0 \text{ and } MV^T = 0\}
	\end{aligned}
	\end{equation}
	where $X=U\Sigma V^T$ is the SVD and $\partial g(\Sigma)$ is the matrix of same size as $\Sigma$ with diagonal elements $\partial g(\sigma_i)$.
	\label{lemma:subgrad}
\end{lemma}

Next we give the stationary point conditions for $\D$ that are needed for proving Theorem ~\ref{thm:dirderiv}.
\begin{lemma}
	Let $B=U\sqrt{\Sigma}$, $C=V\sqrt{\Sigma}$ and $X=U\Sigma V^{\T}$. If $(B,C)$ is a stationary point of $\D$, then
	\begin{align}
	0 &= B \partial G(\Sigma) + \nabla H(BC^T) C, \\
	0 &= \partial G(\Sigma) C^T + B^{\T}\nabla H(BC^T).
	\end{align}
	\label{lemma:statpoints}
\end{lemma}

We are now ready to prove Theorem~\ref{thm:dirderiv}.

\begin{proof}[Proof of Theorem~\ref{thm:dirderiv}]
Let $\bX=\bB\bC^T$, $\tX = \tB\tC^T$  and $\Delta X = \tB\tC^T - \bB \bC^T$. 
We first note that the limit
\begin{equation}
\N'_{\Delta X}(\bX) = \lim_{t \searrow 0}\frac{\N(\bX+t\Delta X) - \N(\bX)}{t},
\label{eq:limit}
\end{equation}
exists since $\N$ is a sum of a finite convex function $G$ and a differentiable function $H$.
Our goal is now to show that the limit is non-negative.
Suppose that we can find a factorization $B(t)C(t)^T = \bX+t \Delta X$, such that $\reg(\bX + t \Delta X ) = \tilde{\reg}(B(t),C(t))$, $(B(t),C(t))$ is continuous and $(B(0),C(0)) = (\bB,\bC)$. Then for small enough $t$ we have 
\begin{equation}
\N(\bX+t\Delta X)- \N(\bX) = \D(B(t),C(t)) - \D(\bB,\bC).
\end{equation}
This quantity is clearly non-negative since $(\bB,\bC)$ is a local minimizer of $\D$, which would prove that the limit \eqref{eq:limit} is non-negative. It is not difficult to see that this can be done when the two matrices $\bX$ and $\tX$ have singular value decompositions with the same $U$ and $V$. In what follows we will first show that all other cases can be reduced so that the matrices are of this form. When this is done we proceed to construct the factorization $B(t)C(t)^T$ which completes the proof.

The directional derivatives can be computed using the sub-differential 
\begin{equation}
\N'_{\Delta X} = \max_{2Z \in \partial G(\bB\bC^T)} \skal{2Z,\Delta X} + \skal{\nabla H(\bB\bC^T), \Delta X}.
\label{eq:subdiffdir}
\end{equation}
By Lemma~\ref{lemma:subgrad}, the first term becomes
\begin{equation}
\begin{aligned}
\skal{U\partial G(\Sigma) V^T + M,\Delta X} &=
\skal{U\partial G(\Sigma)V^T,\tB\tC^T} \\
& +\skal{M,\tB\tC^T} \\
&-\skal{U \partial G(\Sigma) V^T, \bB\bC^T}.
\label{eq:ddir1}
\end{aligned}
\end{equation}
The columns of $\tB$ can be written as a linear combination of the columns in $\bB$ and those of a matrix $\bB_\perp$ with at most~$k$
columns that are perpendicular to~$\bB$.
Similarly, the columns of $\tC$ can be written as a linear combination of the columns in $\bC$ and those of a matrix $\bC_\perp$ with at most $k$ columns that are perpendicular to~$\bC$.
Therefore, we may write
\begin{equation}
\begin{aligned}
\tB\tC^T &= \begin{bmatrix} \bB & \bB_\perp \end{bmatrix} \begin{bmatrix} K_{11} & K_{12} \\ K_{21} & K_{22} \end{bmatrix} \begin{bmatrix}
\bC^T \\
\bC_\perp^T
\end{bmatrix} \\
&= \bB K_{11}C^T+\bB K_{12}\bC_\perp^T \\
&\qquad+\bB_\perp K_{21} \bC^T + \bB_\perp K_{22} \bC_\perp^T,
\end{aligned}
\end{equation}
where $ \bB^T \bB_\perp = 0$ and $\bC^T \bC_\perp = 0$. Our goal is now to show that the terms $K_{12}$ and $K_{21}$ and the off diagonal elements of $K_{11}$ vanish from \eqref{eq:subdiffdir} and can be assumed to be zero.

For the last term of \eqref{eq:ddir1} we have
\begin{equation}
\begin{aligned}
\skal{U \partial G(\Sigma) V^T, \bB\bC^T} &= \skal{\partial G(\Sigma), U^T\bB\bC^TV } \\
&= \skal{\partial G(\Sigma),\Sigma},
\end{aligned}
\end{equation}
which is clearly independent of $\tB$ and $\tC$.
The first term of \eqref{eq:ddir1} reduces to
\begin{equation}
\begin{aligned}
\skal{U\partial G(\Sigma)V^T,\tB\tC^T} &= \skal{U\partial G(\Sigma)V^T,\bB K_{11}\bC^T} \\
&= \skal{\bB^TU\partial G(\Sigma)V^T\bC,K_{11}} \\
&= \skal{\Sigma \partial G(\Sigma),K_{11}}.
\end{aligned}
\end{equation}
Note that the off diagonal elements of $K_{11}$ vanish from this expression since $\Sigma \partial G(\Sigma)$ is diagonal.
Similarly, the second term of \eqref{eq:ddir1} reduces to
\begin{equation}
\skal{M,\tB\tC^T} = \skal{M, \bB_\perp K_{22} \bC_\perp^T}.
\end{equation}
We now consider the second term of~\eqref{eq:subdiffdir}
\begin{equation}
\begin{aligned}
&\skal{\nabla H(\bB\bC^T), \Delta X} =\\
&\qquad\skal{\nabla H(\bB\bC^T), \bB K_{11}\bC^T+\bB K_{12}\bC_\perp^T \\
&\qquad\qquad+ \bB_\perp K_{21} \bC^T + \bB_\perp K_{22} \bC_\perp^T-\bB\bC^T}.
\label{eq:ddir2}
\end{aligned}
\end{equation}
For the first term we have
\begin{equation}
\begin{aligned}
\skal{\nabla H(\bB\bC^T), \bB K_{11}\bC^T} &= \skal{\nabla H(\bB\bC^T)\bC,\bB K_{11}} \\
&= -\skal{\bB\partial G(\Sigma), \bB K_{11}} \\
&= -\skal{\bB^T \bB\partial G(\Sigma), K_{11}} \\
&= -\skal{\Sigma\partial G(\Sigma), K_{11}}.
\end{aligned}
\end{equation}
Again the off diagonal elements of $K_{11}$ vanish.
For the second term of \eqref{eq:ddir2} we have 
\begin{equation}
\begin{aligned}
\skal{\nabla H(\bB\bC^T), \bB K_{12}C_\perp^T} &= \skal{B^T \nabla H(BC^T),K_{12}C_\perp^T}\\
&= -\skal{\partial G(\Sigma)\bC^T, K_{12}\bC_\perp}\\
&=  -\skal{\partial G(\Sigma)\bC^T\bC_\perp, K_{12}} = 0.
\end{aligned}
\end{equation}
Similarly, the third term is $\skal{\nabla H(\bB\bC^T), \bB_\perp K_{21}\bC^T} = 0$. 
Thus
\begin{equation}
\begin{aligned}
\skal{\nabla H(\bB\bC^T),\Delta X} &= \skal{\nabla H(\bB\bC^T),\bB_\perp^T K_{22} \bC_\perp^T}\\
&  -\skal{\Sigma\partial G(\Sigma), K_{11}} \\
&- \skal{\nabla H(\bB\bC^T),\bB\bC^T}.
\end{aligned}
\end{equation}
Summarizing we see that we have now proven that all the terms in \eqref{eq:ddir1} are independent of $K_{12}$, $K_{21}$ as well as the off diagonal terms of $K_{11}$. They therefore do not affect the value of $\N'_{\Delta X}$ and can be assumed to be zero.
We can now write $\Delta X$ as
\begin{equation}
\Delta X = 
\begin{bmatrix}
U & U_\perp
\end{bmatrix}
\begin{bmatrix}
(D-I)\Sigma & 0 \\
0 & \tilde{\Sigma}
\end{bmatrix}
\begin{bmatrix}
V^T \\ 
V_\perp^T
\end{bmatrix},
\end{equation}
where $D$ are the diagonal elements of $K_{11}$ and $U_\perp \tilde{\Sigma} V_\perp^T$ is the SVD of $\bB_\perp K_{22} \bC_\perp^T$. Note that $U_\perp^T U = 0$ since $U$ and $U_\perp$ span orthogonal subspaces. Similarly $V_\perp^T V = 0$.

We now consider the directional derivative \eqref{eq:limit} with $\bB= U\sqrt{\Sigma}$, $\bC=V\sqrt{\Sigma}$.
It is clear that for small $t$ the matrix $\bX + t \Delta X$ has the singular value decomposition
\begin{equation}
\begin{bmatrix}
U & U_\perp
\end{bmatrix}
\begin{bmatrix}
((1-t)I+tD)\Sigma & 0 \\
0 & t\tilde{\Sigma}
\end{bmatrix}
\begin{bmatrix}
V^T \\ 
V_\perp^T
\end{bmatrix}.
\end{equation}
We now let
\begin{align}
B(t) &= \begin{bmatrix}
U & U_\perp
\end{bmatrix}
\sqrt{
\begin{bmatrix}
((1-t)I+tD)\Sigma & 0 \\
0 & t\tilde{\Sigma}
\end{bmatrix}},
\\
C(t) &= \begin{bmatrix}
V & V_\perp
\end{bmatrix}
\sqrt{
	\begin{bmatrix}
((1-t)I+tD)\Sigma & 0 \\
	0 & t\tilde{\Sigma}
	\end{bmatrix}}.
\end{align}
Then, we clearly have $\tilde{\reg}(B(t),C(t)) = \reg(X+t\Delta X)$ for small enough $t$,
which completes the proof.
\end{proof}

Next we will prove Theorem~\ref{thm:lowrank-opt}.
Our results build on those of~\cite{olsson-etal-iccv-2017} and we remind the reader that we exclusively use
$f_\mu(\sigma) = \mu - \max(\sqrt{\mu}-\sigma,0)^2$ throughout this section, but suppress the subscript~$\mu$.
We will use the fact that the directional derivatives in a local minimum are non-negative for all low rank directions to show that $(\bB,\bar{C})$ minimizes the non-convex $\N$ over matrices of $\rank < k$ in Theorem~\ref{thm:lowrank-opt}. For this we will need the following result:

\begin{lemma}\label{lemma:lowrankopt}
	If $\bar{X}$ is a solution to
	$\min_{\rank(X) \leq k}\C(X)$ with $\rank(\bar{X})<k$ 
	and the singular values of $Z$ fulfill $\sigma_i(Z) \notin [(1-\delta_{2k})\sqrt{\mu}, \frac{\sqrt{\mu}}{(1-\delta_{2k})}]$
	then $\bar{X}$ also solves $\min_X\C(X)$.
\end{lemma}

\begin{proof}[Proof of Lemma~\ref{lemma:lowrankopt}]
By von Neumann's trace theorem it is easy to see that the problem $\min_{\rank(X) \leq k}\C(X)$ reduces to a minimization over the singular values of $X$.
We should thus find $\sigma_i(X)$ such that
\begin{equation}
\sum_{i=1}^n \underbrace{- \max(\sqrt{\mu} - \sigma_i(X),0)^2 + (\sigma_i(X)-\sigma_i(Z))^2}_{:=g_i(\sigma_i(X))}
\end{equation}
is minimized and at most $k$ singular values are non-zero.
The unconstrained minimizers of $g_i$ can be written down in closed form:
If $0 \leq \sqrt{\mu} < \sigma_i(Z)$ then $\sigma_i(X) = \sigma_i(Z)$ is optimal giving $g_i(\sigma_i(X)) = 0$.
If $0 \leq \sigma_i(Z) < \sqrt{\mu} $ then $\sigma_i(X) = 0$ is optimal giving $g_i(\sigma_i(X)) = -\mu+\sigma_i(Z)^2$.
Hence for any solution of $\min_{\rank(X) \leq k} \C(X)$ we have $\sigma_i(X) = 0$ if 0 $\leq \sigma_i(Z)\leq \sqrt{\mu}$.
There are now two cases:
\begin{enumerate}
\item If $\sigma_{k+1}(Z) < \sqrt{\mu}$ then the sequence of unconstrained minimizers has at most $k$ non-zero values.
Thus, in this case the resulting $X$ solves both $\min_X \C(X)$ and $\min_{\rank(X) \leq k} \C(X)$.

\item If $\sigma_{k+1}>\sqrt{\mu}$ we will not be able to select $\sigma_i(X) = \sigma_i(Z)$ for all $i$ where $0 \leq \sqrt{\mu} < \sigma_i(Z)$. Choosing $\sigma_i(X) = 0$ gives $g_i(0)=-\mu+\sigma_i(Z)^2<0$. Since $\sigma_i(Z)$ is decreasing with $i$ it is clear that the smallest value is obtained when selecting $\sigma_i(X) = \sigma_i(Z)$ for $i = 1,...,k$.
\end{enumerate}
We now conclude that if $\rank(\bar{X}) < k$ then we are in case~1 and therefore $\bar{X}$ solves the unconstrained problem.
\end{proof}

We are now ready to give the proof of Theorem~\ref{thm:lowrank-opt}.

\begin{proof}[Proof of Theorem~\ref{thm:lowrank-opt}]
Since $\C$ and $\N$ has the same subdifferential (see \cite{larsson-olsson-ijcv-2016}) at $\bX=\bar{B} \bar{C}^T$ it is clear that the directional derivatives 
$\C'_{\Delta X}(\bar{X}) = \N'_{\Delta X}(\bar{X}) \geq 0$, where $\Delta X = \tilde{X}-\bar{B}\bar{C}^T$ and $\rank(\tilde{X})\leq k$. By convexity of $\C$ it is then also clear that 
\begin{equation}
\bar{B}\bar{C}^T \in \argmin_{\rank(X)\leq k} \C(X).
\end{equation}
Since $\rank(\bar{B}\bar{C}^T) < k$, $\bar{B}\bar{C}^T$ is also the unrestricted global minimizer of $\C(X)$ according to Lemma \ref{lemma:lowrankopt}.
By Lemma 3.1 of \cite{olsson-etal-iccv-2017} it is then a stationary point of
$\N(X)$. 

What remains now is to prove that $\bar{X} = \bar{B} \bar{C}^T$ is a global minimizer of $\N$ over all line segments $\bar{X} + t\Delta X$.
This can be done by estimating the growth of the directional derivatives along such lines. 
For this purpose we consider the functions $G$ and $H$ defined as in \eqref{eq:Gdef} and \eqref{eq:Hdef}.
Note that $\bar{X}$ is a stationary point of $\N(X) = G(X)+H(X)$ if and only if $-\nabla H(\bar{X}) = 2Z \in \partial G(\bar{X})$.

Since $\nabla H(\bar{X}+t \Delta X)-\nabla H(\bar{X}) = t \nabla H(\Delta X) = 2t(\A^*\A\Delta X-\Delta X)$ we have
\begin{equation}
\begin{split}
\skal{\nabla H(\bar{X}+t \Delta X)-\nabla H(\bar{X}), t\Delta X} = \\2t^2 (\|\A \Delta X\|^2-\|\Delta X\|_F^2),
\end{split}
\end{equation}
and due to RIP $\|\A \Delta X\|^2-\|\Delta X\|_F^2 \geq -\delta_{2r}\|\Delta X\|^2$.
From Corollary 4.2 of \cite{olsson-etal-iccv-2017} we see that for any $2Z' \in \partial G(\bar{X} + t\Delta X)$ we have 
\begin{equation}
\skal{Z'-Z,t\Delta X} >  t^2 \delta_{2r}\|\Delta X\|_F^2,
\end{equation}
as long as $t \neq 0$.
Since $G'_{\Delta X}(X) = \max_{2Z \in \partial G(X)}\skal{2Z,\Delta X}$, $H'_{\Delta X}(X) = \skal{\nabla H(X),\Delta X}$ and 
$2Z + \nabla H'(\bar{X}) = 0$ we get
\begin{equation}
\N'_{\Delta X}(\bar{X}+t\Delta X) \geq 
\skal{2Z'+\nabla H(\bar{X}+t\Delta X),\Delta X} > 0
\end{equation}
This shows that $\bX$ solves \eqref{eq:regminprobl}.
That $\bX$ also solves \eqref{eq:rankminprobl} is now a consequence of the fact that $\reg(X) \leq \mu \rank(X)$ with equality if $X$ have no singular values in the interval $(0,\sqrt{\mu}]$.
Note that $\bX$ is the unrestricted minimizer of $\C(X)$, where the singular values of $Z$ fulfill $\sigma_i(Z) \notin \left[(1-\delta_{2k})\sqrt{\mu}, \frac{\sqrt{\mu}}{1-\delta_{2k}}\right]$. Since the solution to this problem is hard thresholding $\bX$ has no singular values in $\left(0, \frac{\sqrt{\mu}}{1-\delta_{2k}}\right] \supset (0,\sqrt{\mu}]$.
\end{proof}

For completeness we give the proofs that were previously omitted.
\begin{proof}[Proof of Lemma~\ref{lemma:subgrad}]
With some abuse of notation we define the function $g:\mathbb{R}^n \rightarrow \mathbb{R}$ by $g({\bf{x}}) = \sum_{i=1}^n g(x_i)$, where $x_i$, $i=1,...,n$ are the elements of $\bf{x}$ and $g(x) = f(|x|)+x^2$.
The function $g$ is an absolutely symmetric convex function and $G$ can be written $G(X) = g\circ \sigmavec(X)$, where $\sigmavec(X)$ is the vector of singular values of $X$.
Then according to \cite{lewis1995convex} the matrix $Y\in\partial G(X)$ if and only if 
$Y = U'  \diag(\partial g \circ \sigmavec(X) ) V'^T$ when $X = U'  \diag(\sigmavec(X) ) V'^T$.
(Here we use the full SVD with square orthogonal matrices $U'$ and $V'$.)
Now given a thin SVD $X=U\Sigma V^T$ all possible full SVD's of $X$ can be written
\begin{equation}
X = \begin{bmatrix}
U & U_\bot
\end{bmatrix}
\begin{bmatrix}
\Sigma & 0 \\
0 & 0
\end{bmatrix}
\begin{bmatrix}
V^T \\ V^T_\bot
\end{bmatrix},
\end{equation}
where $U_\perp$ and $V_\perp$ are singular vectors corresponding to singular values that are zero.
Note that $U_\perp$ and $V_\perp$ are not uniquely defined since their corresponding singular values are all zero.
Therefore we get
\begin{equation}
\begin{aligned}
Y &= \begin{bmatrix}
U' & U_\bot
\end{bmatrix}
\begin{bmatrix}
\partial g(\Sigma) & 0 \\
0 & D
\end{bmatrix}
\begin{bmatrix}
V'^T \\ V^T_\bot
\end{bmatrix} \\
&= U' \partial g(\Sigma)V'^T+ U_\perp D V_\perp^T,
\end{aligned}
\end{equation}
where $D$ is a diagonal matrix with elements in $2\sqrt{\mu}[-1,1]$. 
It is clear that $\sigma_1(U_\perp D V_\perp^T)= \sigma_1(D) \leq 2\sqrt{\mu}$.
Furthermore, since $U_\perp$ and $V_\perp$ can be any orthogonal bases of the spaces perpendicular to the column and row spaces of $X$, it is clear that any matrix $M$ fulfilling $U^T M = 0$, $M V = 0$ and $\sigma_1(M) \leq 2\sqrt{\mu}$ can be written $M = U_\perp D V_\perp^T$, hence
	\begin{equation}
\begin{aligned}
\partial G(X) &= \{ U\partial g(\Sigma) V^{\T} + M\;:\;\sigma_1(M)\leq 2\sqrt{\mu} ,\\
&\qquad\qquad\qquad\;U^{\T}M=0,\;MV=0 \}.
\end{aligned}
\end{equation}

\end{proof}

\begin{proof}[Proof of Lemma~\ref{lemma:statpoints}]
The gradients of $\tilde{G}$ are given by
\begin{equation}
\nabla_B \tilde{G}(B,C) = \nabla_B(\tilde{\reg}(B,C))+\nabla_B (\|BC^T\|_F^2).
\end{equation}
For the first term we get
\begin{equation}
\nabla_{B_i} \tilde{\reg}(B,C) = f'\left(\frac{\|B_i\|^2+\|C_i\|^2}{2}\right)B_i.
\end{equation}
With $B = U\sqrt{\Sigma}$ and $C = V\sqrt{\Sigma}$ we get
\begin{equation}
\nabla_B \tilde{\reg}(B,C) = B \begin{bmatrix}
f'(\sigma_1) & 0 & \hdots \\
0 & f'(\sigma_2) &  \hdots \\
\vdots & \vdots & \ddots
\end{bmatrix} = B f'(\Sigma),
\end{equation}
which gives
\begin{equation}
\nabla_B \tilde{G}(B,C) = B f'(\Sigma) + 2BC^TC = B(f'(\Sigma)+2\Sigma).
\end{equation}
For a non-zero $\sigma$ we have $\partial g(\sigma)=\{f'(\sigma) + 2\sigma\}$ and therefore
\begin{equation}
\nabla_B \tilde{G}(B,C) = B(\partial G (\Sigma)),
\end{equation}
where $g(X) = \reg_{\mu}(X)+\|X\|_F^2$.
Similarly we get
\begin{equation}
\nabla_C \tilde{G}(B,C) = C(\partial G(\Sigma)).
\end{equation}
If $(B,C)$ is a stationary point then
\begin{align}
0 &= B \partial G(\Sigma) + \nabla H(BC^T) C, \\
0 &= C \partial G(\Sigma) + (\nabla H(BC^T))^T B. 
\end{align}
The second equation can be re-written to the form stated in the lemma.
\end{proof}

\section{Implementation Details}\label{sec:implementationdetails}
In this section we present some more details on our Iteratively Reweighted VarPro approach.
Recall that our approach consists of three main steps.
In the first step we make a quadratic approximation \eqref{eq:weighedapprox} of the regularization term by replacing $\tilde{\reg}(B,C)$ with 
$\sum_{i=1}^k w_i^{(t)}\left(\|B_i\|^2+\|C_i\|^2\right)$ as described in Section~\ref{sec:implement}. 

In the second step we apply one step of VarPro with the Ruhe Wedin approximation, see~\cite{hong-etal-cvpr-2017} for details on the implementation.
VarPro uses Jacobians with respect to both the $B$ and $C$ parameters.
In our case we have two terms that needs to be linearized.
The regularization term can be written
\begin{equation}
\|\diag(w^{(t)}) B\|_F^2 + \|\diag(w^{(t)}) C\|_F^2,
\end{equation}
where $\diag(w^{(t)})$ is a diagonal matrix with the weights $w_i^{(t)}$ in the diagonal.
The residuals $\diag(w^{(t)}) B$ are already linear and by column stacking the variables we can write them as $J^\textsf{reg}_B \vec{b}$, where $\vec{b}$ is a column stacked version of~$B$. If $B$ has~$k$ columns the matrix $J^\textsf{reg}_B$ will consist of~$k$ copies of the matrix $\diag(w^{(t)})$. Additionally, each row of $J^\textsf{reg}_B$ has only one non-zeros element making the matrix
extremely sparse. Similarly, we obtain the contribution due to the second bilinear factor$C$, which can be written as $J^\textsf{reg}_C \vec{c}$. Here we use $\vec{c}=\text{vec}(C^\T)$, as it
alleviates the computations of the data terms, hence $J^\textsf{reg}_C$ consists of a $k$ copies
of $\diag(w^{(t)})$ permuted to match this design choice.
Given a current iterate $(\vec{b}^{(t)},\vec{c}^{(t)})$ we write the regularization term as
$\|J_B^{\textsf{reg}}\delta \vec{b} + \vec{r}_B\|^2+\|J_C^{\textsf{reg}}\delta \vec{c} + \vec{r}_C\|^2$, where $\vec{r}_B = J_B^{\textsf{reg}} \vec{b}^{(t)}$, $\vec{r}_C = J_C^{\textsf{reg}} \vec{c}^{(t)}$, 
$\vec{b} = \vec{b}^{(t)}+\delta \vec{b}$ and $\vec{c} = \vec{c}^{(t)}+\delta \vec{c}$.

Linearizing the residuals $\A B C^T - b$ around $(\vec{b}^{(t)},\vec{c}^{(t)})$ gives an expression of the form
\begin{equation}
J_B^{\textsf{data}}\delta \vec{b} + J_C^{\textsf{data}}\delta \vec{c} + \vec{r}^{\textsf{data}}.
\end{equation}
The particular shape of the Jacobians in this expression depends on the application; however, in all of our applications they are sparse. 
For example, in the missing data problem each residual corresponds to an element of the matrix $X$ which in turn only depends on $k$ elements of $B$ and $C$.
Locally we may now write the objective function as
\begin{equation}
\|J_B \delta \vec{b} + J_C \delta \vec{c} + \vec{r}\|^2,
\label{eq:LMsys}
\end{equation}
where 
\begin{equation}
\small
J_B = \left[\begin{array}{c}
J_B^\textsf{reg}  \\
0   \\
J_B^\textsf{data} \\
\end{array}\right],
\
J_C = \left[\begin{array}{ccc}
0 \\
J_C^\textsf{reg} \\
J_C^\textsf{data} \\
\end{array}\right],
\
\vec{r} = 
\left[
\begin{array}{c}
\vec{r}_B \\
\vec{r}_C \\
\vec{r}^{\textsf{data}}
\end{array}
\right].
\end{equation}
It was shown in \cite{hong-etal-eccv-2016} that each step of VarPro is equivalent to first minimizing \eqref{eq:LMsys} with the additional dampening term $\lambda \|\delta \vec{b}\|^2$ and then performing an exact optimization of \eqref{eq:weighedapprox} over the $C$-variables (when fixing the $B$-variables to their new values). 
Since we also have a reweighing we only do one iteration with VarPro before updating the weights $w^{(t)}$. 

The above procedure can return stationary points for which $\tilde{\reg}(B,C) > \reg(B C^T)$. Our last step is designed to escape such points by taking the current iterate and recompute the factorization of $\bar{B}\bar{C}^T$ using SVD. If the SVD of $\bar{B}\bar{C}^T= \sum_{i=1}^r \sigma_i U_i V_i^T$ we update $\bar{B}$ and $\bar{C}$ to $\bar{B}_i = \sqrt{\sigma_i}U_i$ and $\bar{C}_i = \sqrt{\sigma_i}V_i$ which we know reduces the energy and gives $\tilde{\reg}(\bar{B},\bar{C}) = \reg(\bar{B}\bar{C}^T)$.
Therefore we proceed by refactorizing the current iterate using SVD in each iteration.
The detailed steps of the bilinear method are summarized in Algorithm~\ref{alg:1}.

\begin{algorithm}[h]
\small
\SetAlgoLined
\DontPrintSemicolon
 \KwIn{Robust penalty function $f$, linear operator $\mathcal{A}$ and regularization parameter~$\mu$,
damping parameter~$\lambda$.}
 Initialize $B$ and $C$ with random entries\;
 \While{not converged}{
  Compute weights $w^{(t)}$ from current iterate $(B,C)$\;
  Compute the vectorizations $\vec{b} = \text{vec}(B)$,  $\vec{c} = \text{vec}(C^\T)$\;
  Compute residuals $\vec{r}_B$ $\vec{r}_C$, and Jacobians $J_{B}^\textsf{data}$ and $J_{B}^\textsf{data}$ depending on~$\mathcal{A}$\;
  Compute residual $\vec{r}^\textsf{reg}$, and Jacobians $J_{B}^\textsf{reg}$ and $J_{C}^\textsf{reg}$\;
Create full residual $\vec{r}$ and Jacobians $J_{B}$ and $J_{C}$\;

    Compute $\tilde{J}^{\T}\tilde{J}+\lambda I  = J_{B}^{\T}(I-J_{C}J_{C}^{+})J_{B}+\lambda I$\;
    Compute $\vec{b}' = \vec{b} - (\tilde{J}^{\T}\tilde{J}+\lambda I)^{-1}J_{B}r$ and reshape into matrix $B'$\;
    Compute $C'$ by minimizing \eqref{eq:weighedapprox} with fixed $B'$\;

  \eIf{$\reg(B'{C'}^T)+\|\A (B'{C'}^T) - b\|^2<\reg(B{C}^T)+\|\A (B{C}^T) - b\|^2$}{
  		$[U,\Sigma,V] = \text{svd}(B'{C'}^T)$\;
        Update $B=U\sqrt{\Sigma}$ and $C=V\sqrt{\Sigma}$\;
        Decrease $\lambda$\;
   }{
    Increase $\lambda$\;
  }
 }
\caption{Outline of the bilinear method.}
\label{alg:1}
\end{algorithm}

\section{Additional Experiments on Real Data}\label{sec:moreexp}

\subsection{pOSE: Psuedo Object Space Error}
In this section we compare the energies over time for ADMM optimizing the same
energy~\cite{larsson-olsson-ijcv-2016}, \ie{}~with the regularizer $\reg$, and $f=f_\mu$
as in~\eqref{eq:fmu} (of the main paper), and our proposed method.
We let the bilinear method run until convergence, and let ADMM execute the same time
in seconds. As a comparison we use the nuclear norm relaxation and the discontinuous rank
regularization. The results of the experiment are shown in Figure~\ref{fig:pose1}.

Again, note that the bilinear method optimizes the same energy as ADMM-$\reg_\mu$, and that,
despite the initial fast lowering of the objective value, the ADMM approach fails to
reach the global optimum, within the allotted 150 seconds.
This holds true for all methods employing ADMM.
In all experiments, the control parameter $\eta=0.5$, and the~$\mu$ parameter
was chosen to be smaller than all non-zero singular values of the best known
optimum (obtained using VarPro).
For a fair comparison, the $\mu$-value for the nuclear norm relaxation,
was modified due to the shrinking bias, and was chosen to be the smallest value
of~$\mu$ for which a solution with accurate rank was obtained. Due to this modification,
the energy it minimizes is not
directly correlated to the others, but is shown for completeness.
Furthermore, the iteration speed of ADMM is significantly
faster than for VarPro, and therefore we show the elapsed time (in seconds) for all methods.
The reported values are averaged over 50 instances with random initialization.

\begin{figure}[h!]
\centering
\includegraphics[width=0.495\textwidth]{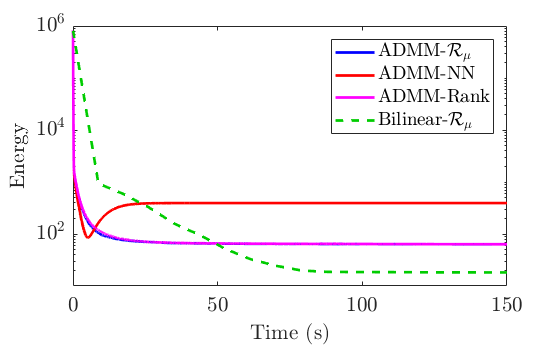}\\
\caption{The average energy for the pOSE problem over 50 instances with random initializations,
for test sequence \emph{Door}. (Note that the energy for ADMM-Rank and ADMM-$\reg_\mu$ are very similar).}
\label{fig:pose1}
\vspace{-0.2cm}
\end{figure}

\subsection{Background Extraction}
The missing data problem formulation can also be used in \eg{}~background extraction, where the goal
is to separate the foreground from the background in a video sequence.
For this experiment, security footage of an airport is used.
The frame size is $144\times 176$ pixels, and we use
the first 200 frames, as in~\cite{he-etal-cvpr2012}.
The camera does not move, hence the background is static.

By
concatenating the vectorization of the frames into a matrix we expect it to be additively
decomposable in terms of a low rank matrix (background) and a
sparse matrix (foreground).
We follow the setup
used in~\cite{cabral-etal-iccv-2013}, and crop the width to half of the height,
and shift it 20 pixels to the right after 100 frames
to simulate a virtual pan of the camera.
This increases the complexity of the background,
as it is no longer static.
Lastly, we randomly drop 70 \% of the entries.
To allow for smaller singular values, we use Geman, as it is a robust penalty
with shrinking bias. The results are shown in Figure~\ref{fig:virtualpan}.

\begin{figure}[h!]
\centering
\includegraphics[width=0.495\textwidth]{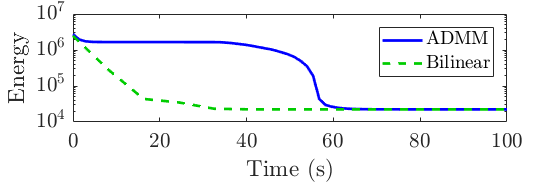}
\caption{Energy minimization comparison for the background extraction experiment.}
\label{fig:background_comp}
\end{figure}

Initially ADMM struggles to find the correct balance between lowering the
rank and fitting the data, which is seen in Figure~\ref{fig:background_comp}, where the objective
is almost unaffected the first forty seconds.
At this point, the bilinear method has already converged.

\begin{figure*}[t]
\centering
\setlength\tabcolsep{0.03575cm}
\def\w{23.1mm}
\def\arraystretch{0.5}
\begin{tabular}{cccccc}
\includegraphics[width=\w]{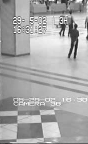} &
\includegraphics[width=\w]{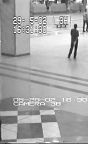} &
\includegraphics[width=\w]{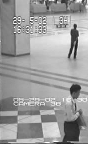} &
\includegraphics[width=\w]{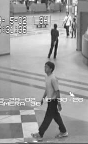} &
\includegraphics[width=\w]{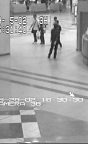} &
\includegraphics[width=\w]{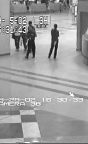}
\\
\includegraphics[width=\w]{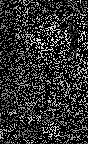} &
\includegraphics[width=\w]{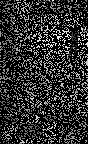} &
\includegraphics[width=\w]{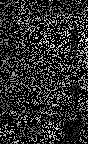} &
\includegraphics[width=\w]{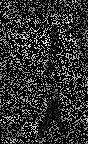} &
\includegraphics[width=\w]{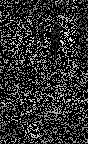} &
\includegraphics[width=\w]{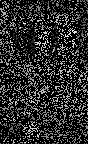}
\\
\includegraphics[width=\w]{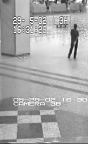} &
\includegraphics[width=\w]{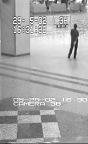} &
\includegraphics[width=\w]{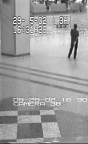} &
\includegraphics[width=\w]{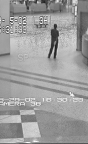} &
\includegraphics[width=\w]{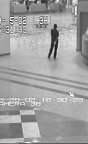} &
\includegraphics[width=\w]{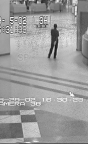}
\end{tabular}
\caption{Background extraction using Geman.
Samples from frame no. 40, 70, 100, 130, 170 and 200.
\emph{Top row:} Original images.
\emph{Middle row:} Training data with 70 \% missing data.
\emph{Bottom row:} Reconstruction of background (bilinear method).}
\label{fig:virtualpan}
\end{figure*}

\begin{figure*}[b]
\centering
\def\w{38mm}
\begin{tabular}{cccc}
\includegraphics[width=\w]{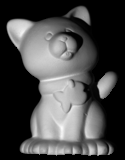} &
\includegraphics[width=\w]{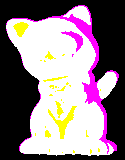} &
\includegraphics[width=\w]{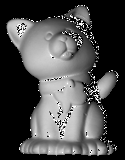} &
\includegraphics[width=\w]{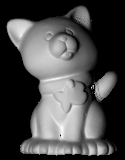}
\\
\includegraphics[width=\w]{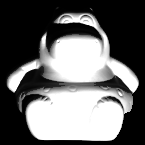} &
\includegraphics[width=\w]{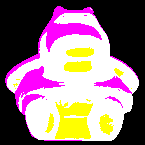} &
\includegraphics[width=\w]{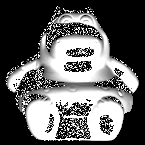} &
\includegraphics[width=\w]{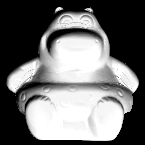}
\\
\includegraphics[width=\w]{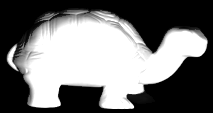} &
\includegraphics[width=\w]{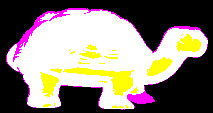} &
\includegraphics[width=\w]{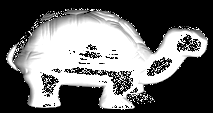} &
\includegraphics[width=\w]{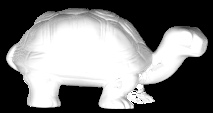}
\end{tabular}
\caption{Images from the photometric stereo experiment. \emph{From left to right}:
(a) Ground truth image,
(b) missing data mask with static background (black), dark pixels (purple), bright pixels (yellow),
(c) reconstruction using ADMM, and
(d) reconstruction using the Bilinear formulation.}
\label{fig:photo_cat}
\end{figure*}

\subsection{Photometric Stereo}
Photometric stereo can be used for estimating depth and surface orientation from images of the
same object and view with varying lighting directions. Assuming~$M$ lighting directions
and~$N$ pixels define $I\in\R^{M\times N}$, where $I_{ij}$ is the light intensity for lighting
direction~$i$ and pixel~$j$. Assuming Lambertian reflectance, uniform
albedo and a distant light source, $I=LN$,
where $L\in\R^{M\times 3}$ contain the lighting directions and~$N\in\R^{3\times N}$ the unknown surface
normals. Thus, the resulting problem is to find a rank 3 approximation of the intensity matrix~$I$.

We use the Harvard Photometric Stereo testset~\cite{Frankot88amethod}, which contains
images of various objects from varying lighting direction.
The images are scaled to $160\times 125$ pixels, and only the foreground
pixels are used in the optimization. Similar to~\cite{cabral-etal-iccv-2013}, we introduce
missing data by thresholding dark pixels with pixel
value less than~40
and bright pixels with pixel value more than~205. The measurement matrix is reconstructed
using the bilinear method and the ADMM equivalent with the~$\reg_\mu$ regularization.
The result is shown in Figure~\ref{fig:photo_cat}.
We let the bilinear method run until convergence and
let the ADMM equivalent run for the same time in seconds, at which point the objective value
is still decreasing when ADMM is interrupted; however, the reduction is almost negligible.
In all cases ADMM fails to converge to
a low rank solution in the same time as the bilinear method, which yields a consistent result.

\clearpage


\end{document}